\documentclass[]{article}

\usepackage{blindtext} %

\usepackage[sc]{mathpazo} %
\usepackage[T1]{fontenc} %
\usepackage{microtype} %

\usepackage[english]{babel} %

\usepackage[hmarginratio=1:1,top=32mm,columnsep=20pt]{geometry} %
\usepackage[hang, small,labelfont=bf,up,textfont=it,up]{caption} %
\usepackage{booktabs} %

\usepackage{lettrine} %

\usepackage{enumitem} %
\setlist[itemize]{noitemsep} %

\usepackage{abstract} %

\usepackage{titlesec} %
\titleformat{\section}[block]{\large\scshape\centering}{\thesection.}{1em}{} %
\titleformat{\subsection}[block]{\large}{\thesubsection.}{1em}{} %

\usepackage{titling} %

\usepackage{hyperref} %

\usepackage{graphicx}
\usepackage{latexsym}

\usepackage{url}
\usepackage{xspace}
\usepackage[colorinlistoftodos,prependcaption,textsize=scriptsize,textwidth=1.6cm]{todonotes}
\usepackage{amsmath}
\usepackage{amsthm}
\usepackage{amsfonts}
\usepackage{amssymb}
\usepackage{listings}
\usepackage{verbatim}
\usepackage[numbers]{natbib}

\usepackage{ textcomp }
\usepackage{tikz}
\usepackage{tensor}
\usetikzlibrary{arrows,matrix,shapes,calc,positioning,shadows,fit}

\newtheorem{corollary}{Corollary}
\newtheorem{lemma}{Lemma}
\theoremstyle{definition}
\newtheorem{definition}{Definition}
\theoremstyle{plain}

\newtheorem{propn}{Proposition}

\theoremstyle{definition}
\newtheorem{example}{Example}

\newcommand{\pci}{C^{\aleph_0}I}
\newcommand{\mci}{C^{m}I}
\newcommand{\sci}{CI}
\newcommand{\infm}{\mathcal{M}_O}
\newcommand{\cin}{\mathcal{N}}
\newcommand{\flip}{\leadsto}
\newcommand{\set}[1]{\{#1\}}
\newcommand{\nrp}[2]{#1++#2}
\newcommand{\nrs}[1]{#1\#\#}
\newcommand{\snrp}[2]{\{#1\}++#2}
\newcommand{\snrs}[2]{\{#1\}\#\#}

\newcommand{\scl}[2]{\{#1\}\leq#2}
\newcommand{\sce}[2]{\{#1\}=#2}
\newcommand{\cg}[2]{#1\geq#2}
\newcommand{\cl}[2]{#1\leq#2}
\newcommand{\ce}[2]{#1=#2}

\newcommand{\M}{\supset}
\newcommand{\cseqo}[1]{\hookrightarrow_{#1}}
\newcommand{\cseqt}[2]{\hookrightarrow_{#1,#2}}

\pretitle{\begin{center}\large\bfseries} %
\posttitle{\end{center}} %
\title{Encoding monotonic multi-set preferences using CI-nets: preliminary report \thanks{Supported by the Austrian Science Fund FWF project W1255-N23.} } %

\author{
\normalsize Martin Diller \\[1ex]
\normalsize Insitute of Information Systems, TU Wien, Vienna, Austria \\
\normalsize \href{mailto:mdiller@kr.tuwien.ac.at}{mdiller@kr.tuwien.ac.at}
\and
\normalsize Anthony Hunter \\[1ex]
\normalsize University College London, London, U.K \\
\normalsize \href{anthony.hunter@ucl.ac.uk}{anthony.hunter@ucl.ac.uk} 
}

\date{} %
\renewcommand{\maketitlehookd}{%
\begin{abstract}
\noindent $CP$-nets and their variants constitute one of the main AI approaches for specifying and reasoning about preferences.  $\sci$-nets, in particular, are a $CP$-inspired formalism for representing ordinal preferences over sets of goods, which are typically required to be monotonic.  

Considering also that goods often come in multi-sets rather than sets,
a natural question is whether $\sci$-nets can be used more or less directly to encode preferences over multi-sets. We here provide some initial ideas on how to achieve this, in the sense that at least a restricted form of reasoning on our framework, which we call ``confined reasoning'',  can be efficiently reduced to reasoning on $\sci$-nets.  Our framework nevertheless allows for encoding preferences over multi-sets with unbounded multiplicities.  We also show the extent to which it can be used to represent 
preferences where multiplicites of the goods are not stated explicitly (``purely qualitative preferences'') as well as a potential use of our generalization of $\sci$-nets as a component of a recent system for evidence aggregation. 
\end{abstract}
}

\begin{document}

\maketitle
\section{Introduction}
$\sci$-nets \cite{BouveretEL09} are part of several languages for specifying and reasoning about preferences that are inspired by CP-nets \cite{BoutilierBDHP04,2011Rossi}  (e.g. \cite{Wilson04,BrafmanDS06,BrafmanDSS06,BrewkaTW10,Wilson11,SanthanamOB13}).  These languages have in common that assertions regarding preferences are interpreted via the ``ceteris-paribus'' or ``all things equal'' semantics.  I.e. ``A is preferred to B'' is interpreted as a shorthand for ``A is preferred to B, ceteris paribus''.  This allows the formulation of an ``operational semantics'' in terms of ``worsening flips'' for verifying statements regarding preferences computationally.  

$\sci$-nets distinguishing feature is that they are a framework tailored particularly for specifying and reasoning about ordinal preferences over sets of goods.  These are also typically monotonic, i.e. more goods are usually preferred to less goods (of the same type).  

Also taking in account the fact that more often than not goods come in multi-sets rather than sets, a natural question is whether $\sci$-nets can be easily generalised to specify and reason about preferences over multi-sets as well as sets of goods.  We here take up this challenge, providing some initial ideas on how to build a framework for the multi-set scenario on top of $\sci$-nets in the sense that at least a restricted form of reasoning on our framework, which we call confined reasoning, can be efficiently reduced to reasoning on $\sci$-nets. 

The framework we propose is based on $\sci$-nets, but can deal with what we identify as the two main differences of preferences over multi-sets and preferences over sets of goods.   The first of the differences  is obviously that, while preferences over sets involve comparing different combinations of a fixed number of elements (namely one of each item), when considering multi-set preferences also the multiplicity of the items needs to be taken in account.  So, for example, while in the set scenario preferring apples over oranges always is interpreted as ``irrespective of the number of apples and oranges'', in the multi-set scenario it is possible to say, for example, that one prefers having an apple over an orange if one doesn't already have any apples, but one prefers having an orange over some number (say, up up to three) apples if one already has some (e.g. two or more) apples.  

A slightly more subtle issue is that, while when talking about preferences over sets there is a natural limit to the number of items one is considering (namely, one of each), in the case of preferences over multi-sets it is often the case that it is artificial to impose any a-priori upper bound on the multiplicity of the items.  For example, when one says that one prefers having an apple and an orange over say even up to three pears, this also means that one prefers having two apples and two oranges over three pears, three apples and one orange over three pears, etc.  If one is using the preferences as a guide as to what choice to take regarding some outcome, e.g. choosing between different baskets of fruits, then the upper bound of apples, oranges, and pears is given by the ``evaluation context'' (in this case, the upper bound of the fruits in the baskets that are available), but is not part of the preference relation per se.  I.e., the same preference relation should be of use when considering a different ``evaluation context'', e.g. a different set of fruit baskets.       

Now, often when talking about preferences over multi-sets the multiplicities of items are not stated explicitely, e.g. one says that one prefers apples over oranges, which often may either mean that one prefers an apple over any (relevant) number of oranges, or that one prefers any number of apples over a comparable number of oranges.  In the same manner, one can say that one prefers oranges over apples if one already has \emph{enough} apples.  We call such preferences \emph{purely qualitative}.  Although in the framework we propose here multiplicities of items are considered explicitely, we also show that certain basic forms of purely qualitative preferences can be encoded in it in a straightforward manner.

To further motivate our generalization of $\sci$-nets in this work we also give an example of its use in the context of an argument-based system for the aggregation of evidence stemming from clinical trials \cite{HunterW12,HunterW15}.  Specifically, we show how it can be applied to order the available evidence, which is then subject to further critical analysis by the system, based on personalized criteria.

As to the structure of this paper, we start of by giving the relevant background on $\sci$-nets in Section \ref{sec::ci}.  In Section \ref{sec:pci} we then present our framework which, following the practice of \citeauthor{SanthanamOB13} (\citeyear{SanthanamOB13}) (who present a variant of $\sci$-nets for representing preferences among sets of countermeasures), we call $\pci$-nets (the $\aleph_0$ standing for the fact that our generalisation of $\sci$-nets is designed for stating preferences over multi-sets with unbounded multiplicities).  We here also discuss how to encode certain forms of purely qualitative preferences over multi-sets via $\pci$-nets and characterise reasoning about $\pci$-nets in terms of the above mentioned confined reasoning.   

In Section \ref{sec:mci} we present a generalisation of $\sci$-nets for encoding preferences over multi-sets with \emph{bounded} multiplicities, which forms the basis of the reduction of confined reasoning for $\pci$-nets to reasoning on $\sci$-nets that we also develop in this section.  In Section \ref{sec:app} we show how $\pci$-nets can be applied as a component of the before-mentioned system for evidence aggregation. Section \ref{sec:conc} presents our conclusions and future work.

\section{Background: $\sci$-nets.}
\label{sec::ci}

We begin by presenting $\sci$-nets.  Throughout this work we consider $O$ to be a fixed finite set of objects, items or goods.  $\sci$-nets consist in a set of $\sci$-statements.  

\begin{definition}[$\sci$-statement] \label{def::scistmnt}
A \emph{conditional importance statement ($\sci$-statement) on $O$} is an expression of the form  
\begin{align*}
S^+,S^- : S_1 \vartriangleright S_2
\end{align*}
\noindent where $S^+ , S^- , S_1 , S_2$ are pairwise disjoint subsets of $O$, while $S_1,S_2$ are non-empty.
\end{definition}
  
\noindent The informal interpretation of a $\sci$-statement $S^+,S^- : S_1 \vartriangleright S_2$ is: ``if I have all the items in $S^+$ and none of those in $S^-$ , I prefer obtaining all items in $S_1$ to obtaining all those in $S_2$ , ceteris paribus''. $S^+$ and $S^-$ are called the positive precondition and the negative precondition respectively. $S_1$ and $S_2$ are called the compared sets.
As stated before, a \emph{$\sci$-net} on $O$ is then a set $\mathcal{N}$ of $\sci$ statements on $O$.

A (strict) \emph{preference relation} is a strict partial order (an irreflexive, asymmetric and transitive binary relation) over $2^{O}$. A preference relation is \emph{monotonic} if $S_a \supset S_b$ entails $S_a > S_b$ for any $S_a, S_b \in 2^{O}$.  The formal semantics of $\sci$ statements are, as to be expected, given in terms of monotonic preference relations over $2^{O}$.

\begin{definition}[Semantics of $\sci$-statements] A preference relation over $2^{O}$ \label{def:45} satisfies a $\sci$-statement $S^+ , S^- : S_1 \vartriangleright S_2$ if for every $S' \subseteq (O \setminus (S^+ \cup S^- \cup S_1 \cup S_2))$, we have $(S' \cup S^+ \cup S_1) > (S' \cup S^+ \cup S_2)$.
\end{definition}

A preference relation over $2^{O}$ then satisfies a $\sci$-net $\mathcal{N}$ if it satisfies each $\sci$-statement in $\mathcal{N}$ and is monotonic. A CI-net $\mathcal{N}$ is \emph{satisfiable} if there exists a preference relation satisfying $\mathcal{N}$. Although there may be several preference relations satisfying a $\sci$-net $\cin$, following \citeauthor{BouveretEL09} we are mainly interested, in the so called ``\emph{induced preference relation}'', which we denote $>_{\mathcal{N}}$.  If $\mathcal{N}$ is satisfiable, this is the smallest preference relation satisfying $\mathcal{N}$.

\begin{example} \label{ex:01} The following is an example of a $\sci$-net from \cite{BouveretEL09}.
\begin{align}
\{a\},\emptyset : \{d\} \vartriangleright \{bc\}; \label{s1}\\
\{a\},\{d\} : \{b\} \vartriangleright \{c\};\label{s2}\\
 \{d\},\emptyset : \{b\} \vartriangleright \{c\} \label{s3}
\end{align}
\end{example}

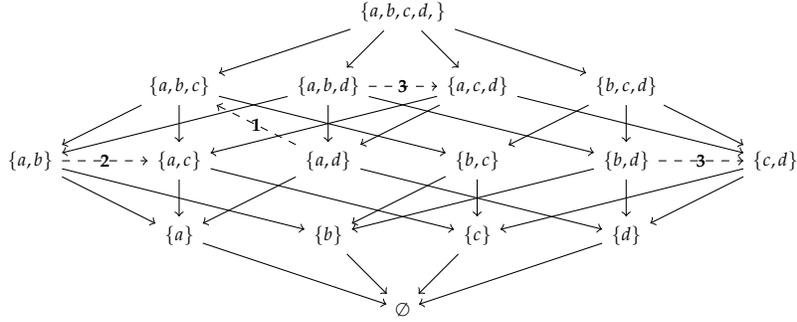
\begin{figure}[t!]
\resizebox{0.7\textheight}{!}{\begin{minipage}{\textwidth}
\centering
\begin{tikzpicture}[main node/.style={fill=none,font=\scriptsize}]

\node[main node] (abcd) at (0,0) {$\{a,b,c,d,\}$};

\node[main node] (abc) at (-3,-1) {$\{a,b,c\}$};
\node[main node] (abd) at (-1,-1) {$\{a,b,d\}$};
\node[main node] (acd) at (1,-1) {$\{a,c,d\}$};
\node[main node] (bcd) at (3,-1) {$\{b,c,d\}$};

\node[main node] (ab) at (-5,-2) {$\{a,b\}$};
\node[main node] (ac) at (-3,-2) {$\{a,c\}$};
\node[main node] (ad) at (-1,-2) {$\{a,d\}$};
\node[main node] (bc) at (1,-2) {$\{b,c\}$};
\node[main node] (bd) at (3,-2) {$\{b,d\}$};
\node[main node] (cd) at (5,-2) {$\{c,d\}$};

\node[main node] (a) at (-3,-3) {$\{a\}$};
\node[main node] (b) at (-1,-3) {$\{b\}$};
\node[main node] (c) at (1,-3) {$\{c\}$};
\node[main node] (d) at (3,-3) {$\{d\}$};

\node[main node] (e) at (0,-4) {$\emptyset$};

\path

(abd) edge[->,dashed] node{\bf\scriptsize \ref{s3}}(acd)
(ad) edge[->,dashed] node{\bf\scriptsize \ref{s1}}(abc)
(ab) edge[->,dashed] node{\bf\scriptsize \ref{s2}}(ac)
(bd) edge[->,dashed] node{\bf\scriptsize \ref{s3}}(cd)

(abcd) edge[->] (abc)
(abcd) edge[->] (abd)
(abcd) edge[->] (acd)
(abcd) edge[->] (bcd)

(abc) edge[->] (ab)
(abc) edge[->] (ac)
(abc) edge[->] (bc)

(abd) edge[->] (ab)
(abd) edge[->] (ad)
(abd) edge[->] (bd)

(acd) edge[->] (ac)
(acd) edge[->] (ad)
(acd) edge[->] (cd)

(bcd) edge[->] (bc)
(bcd) edge[->] (bd)
(bcd) edge[->] (cd)

(ab) edge[->] (a)
(ab) edge[->] (b)

(ac) edge[->] (a)
(ac) edge[->] (c)

(ad) edge[->] (a)
(ad) edge[->] (d)

(bc) edge[->] (b)
(bc) edge[->] (c)

(bd) edge[->] (b)
(bd) edge[->] (d)

(cd) edge[->] (c)
(cd) edge[->] (d)

(a) edge[->] (e)
(b) edge[->] (e)
(c) edge[->] (e)
(d) edge[->] (e)

;

\end{tikzpicture}
\end{minipage}}
\caption{Graphical representation of the preference relation induced by the $\sci$-net from Example \ref{ex:01}.  The solid arcs in the figure are obtained by monotonicity, dotted ones by $\sci$-statements. Transitivity arcs are omitted.}
\label{fig:12}
\end{figure}

\noindent See Figure \ref{fig:12} for a graphical representation of the preference relation induced by the $\sci$-net from Example \ref{ex:01}. 

An alternative, but equivalent (see \cite{BouveretEL09}), way of interpreting CI-statements is in terms of worsening flips, similar to flipping sequences in CP-nets \cite{BoutilierBDHP04}.  The idea behind this operational semantics is that in order to decide whether $S_a >_{\cin} S_b$ ($S_a$ ``dominates'' $S_b$)  for $S_a,S_b \subseteq O$, one needs to find a \emph{sequence of worsening flips} $S_a=S_1,\ldots,S_n=S_b$  where each worsening flip $S_i \flip S_{i+1}$ ($1 \leq i < n$) corresponds to $S_i >_{\cin} S_{i+1}$ being sanctioned either (i) because $S_i \supset S_{i+1}$ ($\M$ flip), or (ii) by a $\sci$-statement in $\cin$ (CI flip).  Transitivity of the preference relation then allows us to conclude that $S_a >_{\cin} S_b$.  In particular, a CI-net $\mathcal{N}$ is satisfiable if and only if it does not possess any cycle of worsening flips.

\begin{definition}[Worsening flips for $\sci$-nets] \label{def:opsci} Let $\mathcal{N}$ be a CI-net on the set of objects $O$, and let $S_a$ , $S_b$ $\subseteq O$. Then $S_a \flip S_b$ is called a worsening flip wrt. $\mathcal{N}$ if one of the following two conditions is satisfied:
\begin{itemize}
\item $S_a \supset S_b$ ($\M$ flip)
\item there is a CI-statement $S^+,S^- : S_1 \vartriangleright S_2 \in \mathcal{N}$ and $S' \subseteq (O \setminus (S^+ \cup S^- \cup S_1 \cup S_2))$ s.t. $S_a = (S' \cup S^+ \cup S_1)$ and $S_b = (S' \cup S^+ \cup S_2)$ (CI flip).
\end{itemize}
An equivalent but more clearly operational characterisation of the latter condition is that if $\overline{S} = (O \setminus (S^+ \cup S^- \cup S_1 \cup S_2))$, then:
\begin{itemize}
\item $(S_1 \cup S^+) \subseteq S_a$ , $(S_2 \cup S^+) \subseteq S_b$ ;
\item $(S_a \cap S^-) = (S_b \cap S^-) = (S_a \cap S_2) = (S_b \cap S_1) = \emptyset$;
\item and $(\overline{S} \cap S_a) = (\overline{S} \cap S_b)$.
\end{itemize}
\end{definition}

\noindent We denote the fact there exists a sequence of worsening flips from $S_a$ to $S_b$ w.r.t. $\cin$ as $S_a \cseqo{\cin} S_b$.  We will also often say that a flip is w.r.t. the $\sci$-statement that ``justifies'' it.  Then a sequence of flips is w.r.t. a a certain set of $\sci$-statements when the flips in the sequence are w.r.t. the $\sci$-statements in the set.

Dominance in satisfiable CI-nets is PSPACE-complete, even if every CI-statement bears on singletons and has no negative preconditions, every CI-statement bears on singletons and has no positive preconditions or, finally,  every CI-statement is precondition-free.  
 Satisfiability of CI-nets is also PSPACE-complete.  Dominance and satisfiability in precondition-free singletons (so called, ``SCI-nets'') is in P. 

\citeauthor{BouveretEL09} also identify coNP-complete as well as polynomial sufficient conditions to determine whether any $\sci$-net is satisfiable.  For instance, for the polynomial sufficient condition, let the preference graph $G(\cin)$ corresponding to a $\sci$-net $\cin$, be defined as the graph whose directed
edges are the pairs $(o_1, o_2)$ such that there is a $\sci$-statement $S^+,S^- : S_1 \vartriangleright S_2 \in \mathcal{N}$ with $o_1 \in S_1$ and $o_2 \in S_2$.  Then, any $\sci$-net with an acyclic preference graph is satisfiable.

\section{$\pci$-nets}
\label{sec:pci}

In this Section we present $\pci$-nets.  First, we give their definition and extensional semantics, then we consider to what extent purely qualitative assertions regarding multi-set preferences can be encoded via $\pci$-nets.  Finally, we present the operational semantics and, in this context, confined reasoning for $\pci$-nets.  

\subsubsection{Definition \& extensional semantics}

We continue to consider a fixed set $O$ of objects and identify a multi-set $M$ on $O$ via its multiplicity function $m_M$.  The latter associates to each $o \in O$ the number $m_M(o)$ of occurrences or instances of $o$ that are in $M$.  We will often represent a multi-set $M$ in the form $\{(o,m_M(o)) \mid o \in O, m_M(o)\geq1 \}$.  We will also use standard terminology and notation for sets to be interpreted as the corresponding notion for multi-sets.  $\infm$ denotes all finite multisets defined on $O$.

$\pci$-nets consist of a set of $\pci$-statements which, similarily to $\sci$-statements involve a ``precondition'' and a ``comparison expression''.  Preconditions of $\pci$-statements consist of what we call a sequence of \emph{simple constraints}, i.e. expressions of the form
\begin{align*}
o_1 R_1 a_1, \ldots ,o_n R_n a_n  
\end{align*}
\noindent  where $o_i \in O$, $R_i \in \{\geq,\leq,=\}$, the $a_i$ are integers $\geq 0$.  We say that a precondition is defined \emph{on} the objects $o_i$ ($1 \leq i \leq n$) and that a multi-set $M' \in \infm$ \emph{satisfies} the precondition,
\begin{align*}
M' \models o_1 R_1 a_1, \ldots ,o_n R_n a_n,    
\end{align*}
\noindent iff 
\begin{align*}
 m_{M'}(o_i) R_i a_i \textit{ for every } 1 \leq i \leq n.
\end{align*}
\noindent A precondition $P^+$ is \emph{satisfiable} if there is some $M' \in \infm$ s.t. $M' \models P^+$. A precondition consisting of an empty sequence of constraints is satisfiable by any multiset by definition.

Comparison expressions of $\pci$-statements involve \emph{update patterns}.  These are expressions of the form
\begin{align*}
\nrp{o_1}{a_1}, \ldots, \nrp{o_n}{a_n}  
\end{align*}
\noindent with each $o_i \in O$ appearing at most once, the $a_i \geq 1$.  Just as for the preconditions of $\pci$-statements, we say that an update pattern is defined \emph{on} the objects $o_i$ ($1 \leq i \leq n$). Now we define the \emph{update of a multiset $M' \in \infm$ w.r.t. an update pattern} as
\begin{align*}
M' [\nrp{o_1}{a_1}, \ldots, \nrp{o_n}{a_n}] := M''   
\end{align*}
\noindent where $m_{M''}(o) = m_{M'}(o)$ for $o \in O$ but $o \neq o_i$ for every $1 \leq i \leq n$, and $m_{M''}(o_i) = m_{M'}(o_i) + a_i$ for $1 \leq i \leq n$.
\begin{definition}[$\pci$-statement] 
A $\pci$-statement w.r.t. a set of objects $O$ is an expression of the form 
\begin{align*}
P^+:P_1 \vartriangleright P_2
\end{align*}
\noindent where $P^+$ is a possibly empty sequence of constraints on $O$ and $P_1,P_2$ are update patterns defined on non-empty, \emph{disjoint} subsets of the objects $O$.  The $\pci$-statement is \emph{satisfiable} if the precondition $P^+$ is satisfiable. 
\end{definition}
\noindent If $P^+$ is an empty sequence we write $P_1 \vartriangleright P_2$.  We will also often use the shorthand $\{o_1,\ldots,o_n\}Ta$ for any sequence $o_1Ta,\ldots,o_nTa$ with $T \in \{\geq,\leq,=,++\}$.
The intuitive meaning of a $\pci$-statement $P^+:P_1 \vartriangleright P_2$ with 
\begin{align*}
P^+ &= \{o_i^+ R_i^+ a_i^+\}_{1 \leq i \leq n_{+}}, \\ 
P_1 &= \{\nrp{o_j^1}{a_j^1}\}_{1 \leq j \leq n_1}, \\
P_2 &= \{\nrp{o_k^2}{a_k^2}\}_{1 \leq k \leq n_2} 
\end{align*}
\noindent is: ``if I have  $R_i^+ a_i$ of $o_i$ ($1 \leq i \leq n_{+}$), I prefer having $a_j^1$ more of $o_j^1$ ($1 \leq j \leq n_1$), than having $a_k^2$ more of $o_k^2$ ($1 \leq k \leq n_2$), ceteris paribus''.  The formal semantics of $\pci$-statements are, analogously to $\sci$-statements, defined over preference (binary irreflexive, antisymmetric, and transitive) relations $>$ over $\mathcal{M}_O$.

\begin{definition}[Semantics of $\pci$-statements] \label{def:46} A preference relation $>$ over $\mathcal{M}_O$  \emph{satisfies} a $\pci$ statement $P^+:P_1 \vartriangleright P_2$ if for every $M' \in \mathcal{M}_O$ s.t. $M' \models P^+$, we have $M'[P_1] >  M'[P_2]$ \end{definition}

\noindent Alternatively, by abuse of notation we define  $P^+ := \{M' \in \infm \mid M' \models P^+\}$ and for an update pattern $P = \{\nrp{o_i}{a_i}\}_{1 \leq i \leq n}$, $M_P := \{(o_i,a_i) \mid 1 \leq i \leq n\}$. Then $>$ satisfies the $\pci$-statement $P^+:P_1 \vartriangleright P_2$ if for every $M' \in P^+$, we have $(M' \cup M_{P_1}) > (M' \cup M_{P_2})$.  Note that if $P^+$ is unsatisfiable, then the $\pci$-statement $P^+ : P_1 \vartriangleright P_2$ is trivially satisfied by any preference relation over $\mathcal{M}_O$.

Now, a $\pci$-net on the set of objects $O$ is, as already indicated, a set $\mathcal{N}$ of $\pci$-statements on $O$.  A preference relation $>$ over $\mathcal{M}_O$ \emph{satisfies} a $\pci$-net $\mathcal{N}$ if $>$ satisfies each $\pci$ statement in $\mathcal{N}$ and $>$ is monotonic (on $\infm$). Finally, a $\pci$-net $\mathcal{N}$ is \emph{satisfiable} if there exists a preference relation $>$ satisfying $\mathcal{N}$. Just as for $\sci$-nets we are mainly interested in the smallest \emph{preference relation induced} by a satisfiable $\pci$ net $\cin$, which we also denote $>_{\cin}$.  The latter relation exists since the intersection of preference relations satisfying $\cin$ also satisfies $\cin$.

\begin{example}
\label{ex:001}
Let $\cin$ be the $\pci$-net
\begin{align}
\nrp{a}{1} \vartriangleright \snrp{b,c,d}{6}; \label{p1} \\  
\cg{a}{1} : \nrp{b}{1} \vartriangleright \snrp{c,d}{3}; \label{p2}\\ 
\cg{a}{3},\cl{b}{2} : \nrp{c}{3} \vartriangleright \nrp{d}{3} \label{p3} 
\end{align}
\noindent We will later be able to show that $\cin$ is satisfiable (see Example \ref{ex:76}). From this specification it follows that, e.g. 
\begin{align*}
\{(a,3),(b3)\} >_{\cin} \{(a,3),(b,2),(d,5)\}.
\end{align*} 
\noindent A proof of this is shown in Example \ref{ex:0001A}, after having introduced the operational semantics for $\pci$-nets.  
\end{example}

Given the size of $\mathcal{M}_O$, we cannot even hope to give a graphical representation of the induced preference relation of the $\pci$-net in Example \ref{ex:001} analogous to that of the induced preference relation of the $\sci$-net in Example \ref{ex:01} (Figure \ref{fig:12}).  As hinted at in the introduction to this work, we will nevertheless later on define 
a form of restricted reasoning on $\pci$ nets for which a graphical representation is at least in theory possible (note that in practice even graphical representations of most $\sci$-nets will be too large to be of much use).

Proposition 5 in \cite{BouveretEL09} states that all monotonic preferences over $2^O$ can be captured via $\sci$-nets.  The proof of this proposition can be easily adapted to show that all motonic preferences over $\infm$ can be captured via $\pci$-nets.  In effect, consider a preference relation $>$ over $\infm$ and $M_a, M_b \in \infm$ s.t. $M_a > M_b$ but $M_a \not\supset M_b$.  Then there must be some disjoint and non-empty $O_1,O_2 \subseteq O$ s.t. $m_{M_a}(o) > m_{M_b}(o)$ for $o \in O_1$, $m_{M_b}(o) > m_{M_a}(o)$ for $o \in O_2$, while $m_{M_a}(o) = m_{M_b}(o)$ for $o \in (O \setminus (O_1 \cup O_2))$   Now the $\pci$-statement $P^+:P_1 \vartriangleright P_2$ with 
\begin{align*}
P^+ &:= \{\ce{o}{m_{M_a \cap M_b}(o)}\}_{o \in O}, \\
P_1 &:= \{\nrp{o}{(m_{M_a}(o) - m_{M_b}(o))}\}_{o \in O_1}, \\
P_2 &:= \{\nrp{o}{(m_{M_b}(o) - m_{M_a}(o))}\}_{o \in O_2} 
\end{align*}
\noindent clearly ``induces'' precisely $M_a > M_b$.  Hence the preference relation induced by the $\pci$-net that consists of the $\pci$-statements obtained in this manner for each such $M_a,M_b$ clearly captures $>$.  The obvious problem with this argument is that, such as the $\sci$-net in the proof of Proposition 5 in \cite{BouveretEL09} can have an exponential number of  $\sci$-statements, the $\pci$-net constructed in the way we have just detailed can have a potentially infinite number of $\pci$-statements\footnote{This, despite the fact that, just as $\sci$-statements can induce up to an exponential number of what we call ``CI flips'' later on, $\pci$-statements have the power to express a potentially infinite number of CI flips.}.  We leave it as an open question whether there is any useful alternative characterisation of the kinds of preference relations that can be captured efficiently (hence, also finitely) by $\pci$ (and, for that manner, $\sci$) nets.  In the rest of this work we nevertheless asume the $\pci$-nets we are dealing with have a finite number of $\pci$-statements.     

\subsubsection{Purely qualitative preferences}

We now briefly consider how purely qualitative preferences as described in the introduction can be encoded using $\pci$ nets.  More to the point, we consider this issue under the assumption that purely qualitative preferences can be interpreted in precise terms, i.e. that they are shorthands for preferences where multiplicites of items are explicitely referred to.

Consider first purely qualitative ``comparison expressions'', for example, an assertion of the form ``I prefer apples over oranges''.  If this is interpreted as an assertion regarding multi-sets rather than sets of fruits, one possible interpretation of this assertion is ``I prefer an apple over an orange''.  This can clearly be encoded easily using the $\pci$-statement
\begin{align*}
\nrp{a}{1} \vartriangleright \nrp{o}{1}
\end{align*} 
\noindent Other similar interpretations such as ``I prefer an apple over a relevant/reasonable number of oranges'' or ``I prefer a sufficient number of apples over a relevant/reasonable number of oranges'' can be encoded in a similar manner.  ``I prefer an apple over an orange'' could also mean ``I prefer a sufficent  number $R$ (e.g. $R = 1$) of apples over \emph{any} number of oranges''.  Although $\pci$-nets don't have any means of expressing ``any'', we can introduce the special symbol $\nrs{}$ to be interpreted as the maximum number of each object given in a specific context (this will become clearer later on).  We also use just $\nrs{o}$ as a shorthand for $\nrp{o}{\nrs{}}$.  Then, the following encoding would do:
\begin{align*}
\nrp{a}{R} \vartriangleright \nrs{o}.
\end{align*}    
\noindent Another possible interpretation of ``I prefer apples over oranges'' is ``I prefer comparable numbers of apples to oranges''.  This assertion can only be encoded as a set of $\pci$-statements:
\begin{align*}
\{\nrp{a}{X} \vartriangleright \nrp{o}{X} \mid X \geq 1 \}.
\end{align*}
\noindent To avoid an infinite number of $\pci$-statements, the following alternative encoding will do for most practical scenarios
\begin{align*}
\{\nrp{a}{X} \vartriangleright \nrp{o}{X} \mid 1 \leq X \leq \nrs{} \}.
\end{align*}
All of the above representations of ``I prefer apples over oranges'' can be easily generalised to statements such as ``I prefer pineapples and mangos to apples and oranges''.  Nevertheless, we note that the generalisation of the last reading we considered would require exponentially many $\pci$-statements, e.g. 
\begin{align*}
\{\nrp{p}{X}&,\nrp{m}{Y} \vartriangleright \nrp{a}{V},\nrp{o}{W} \mid \\
&1 \leq X,Y,V,W \leq \nrs{}, X + Y = V + W \}
\end{align*}
\noindent says that ``I prefer pineapples and mangos to the same number of apples and oranges''. 

Consider now the preconditions of assertions regarding preferences, i.e. the ``if X'' part of an assertion of the form ``if X, I prefer Y over Z''.  Some common pre-conditions can be easily encoded using $\pci$-statements.  E.g. ``if I have a certain number $R$ of apples'' or, similarly, ``if I have enough (i.e. at least $R$) apples'' can clearly be encoded as $\ce{a}{R}$, $\cg{a}{R}$ respectively.  ``If I don't have any apples'' and ``if I don't have enough (i.e. less than $R$) apples'' can be encoded as $\ce{a}{0}$ and $\cl{a}{R-1}$ respectively.  Any combination of such assertions can also be efficiently encoded using preconditions of $\pci$-statements.  Slighly more complicated are assertions of the form ``if I have more apples than oranges''.  Here, again, several $\sci$-statements are necessary:
\begin{align*}
\{\cg{a}{X},\cl{o}{X-1}: \ldots \mid 1 \leq X \leq \nrs{}\}.
\end{align*}      
\noindent Clearly, combinations of such more complicated assertions, e.g. ``if I have more apples than oranges and more pineapples than mangos'', require an exponential number of $\pci$-statements.  

The fact that, in order to encode some purely qualitative assertions regarding preferences, (possibly exponentially many) sets of $\pci$-statements are needed stems from the fact that, as will become clearer in Section \ref{sec:mci}, $\pci$ statements are ultimately based on $\sci$ statements.  In particular, they inherit the ``atomic'' nature of $\sci$ statements.  Consider, for example, that one has disjoint subsets $S,S' \subseteq O$ and one wants to express the fact that any one element from $S$ is preferred to any one element of $S'$.  This can only be written as the set of $\sci$-statements
\begin{align*}
\{o \vartriangleright o' \mid o \in S, o' \in S'\}.
\end{align*}
\noindent An assertion with a precondition stating ``if I have any elements from S and I don't have some elements from $S'$'' requires exponentially many $\sci$ statements:
\begin{align*}
\{T,T': \ldots \mid \emptyset \subset T \subseteq S, \emptyset \subset T' \subseteq S'\}.
\end{align*}

\subsubsection{Operational semantics \& confined reasoning} \label{sub:opp}

We turn to giving an operational semantics for $\pci$-nets analogous to that presented in Section \ref{sec::ci} for $\sci$-nets, i.e. in terms of ``worsening flips''.  This will also be the basis of our definition of confined reasoning for $\pci$-nets.     

\begin{definition}[Worsening flips for $\pci$-nets] \label{def:06}
Let $\cin$ be a $\pci$-net on $O$ and $M_a,M_b \in \infm$. Then $M_a \flip M_b$ is called a \emph{worsening flip} w.r.t. $\cin$ if one of the two following conditions is satisfied:
\begin{itemize}
\item $M_a \supset M_b$ ($\M$ flip)
\item There is a $\pci$ statement $P^+:P_1 \vartriangleright P_2 \in \mathcal{N}$ and an $M' \in \infm$ s.t.  $M' \models P^+$, $M_a = M'[P_1]$, and $M_b = M'[P_2]$ (CI flip).  

Alternatively, $M_a = M' \cup M_{P^C_1}$, $M_b = M' \cup M_{P^C_2}$ for some  $M' \in P^+$ or, operationally:
\begin{itemize}
\item $M_{P_1} \subseteq M_a$,
\item $M_{P_2} \subseteq M_b$,
\item $(M_a \setminus M_{P_1}) =  (M_b \setminus M_{P_2})$, and
\item if $M' = (M_a \setminus M_{P_1}) =  (M_b \setminus M_{P_2})$, then $M' \in P^+$ (i.e. $M' \models P^+$).
\end{itemize}  
\end{itemize}
\end{definition}

\noindent Again, we denote that there exists a sequence of worsening flips from $M_a$ to $M_b$ w.r.t. the $\pci$-net $\cin$ as $M_a \cseqo{\cin} M_b$.  The following proposition can be proven in analogous fashion to the proof of Theorems 7 and 8 in \cite{BoutilierBDHP04}.

\begin{propn}\label{pro:06}  Let $\mathcal{N}$ be a satisfiable $\pci$-net defined on $O$, and $M_a, M_b \in \infm$. We have $M_a >_{\cin} M_b$ iff $M_a \cseqo{\cin} M_b$.
\end{propn}

\begin{example}\label{ex:0001A}
Consider again the $\pci$-net $\cin$ from Example \ref{ex:001}. The following is a sequence of flips from which
\begin{align*}
\{(a,3),(b3)\} >_{\cin} \{(a,3),(b,2),(d,5)\}
\end{align*} 
\noindent can be derived:
\begin{align*}
&\{(a,3),(b3)\}\\
&\flip (CI,\ref{p2}) \\
&\{(a,3),(b,2),(c,3),(d,3)\}\\
&\flip (CI,\ref{p3}) \\
&\{(a,3),(b,2),(d,6)\}\\
&\flip (\M) \\
&\{(a,3),(b,2),(d,5)\}\\
\end{align*}
\noindent The labels of the flips indicate the type of flip and, in case of a CI flip, the $\pci$-statement that justifies the flip.    
\end{example}

From a computational perspective, a major difference between $\sci$-nets and $\pci$-nets on $O$ is that while for the former, sequences of worsening flips without loops can be of length up to $2^{O}$ (given that there are up to this number of distinct subsets of $O$), in the latter sequences without loops of arbitrary length are possible.  Nevertheless, if a sequence of worsening flips from any one multiset $M_a$ to any other multiset $M_b$  exists then it will, by definition, be of finite length.  Stated otherwise, such a sequence will contain only finitely many subsets of $\infm$ and, hence, can be said to be \emph{confined} to any subset $M$ of $\infm$ for which it holds that each multiset $M'$ in the sequence is s.t. $M' \subseteq M$.    

\begin{definition}[Confinement of sequences of worsening flips] \label{def:07} Let $M \in \infm$. We say that a sequence of worsening flips $M_a = M_1 \ldots M_n = M_b$  w.r.t. a $\pci$-net $\cin$ is \emph{confined to} $M$  if each flip $M_i \flip M_{i+1}$ (for $1 \leq i < n$) in the sequence is s.t.  $M_i,M_{i+1} \subseteq M$. We denote $M_a \cseqt{\cin}{M} M_b$ if there is a a sequence of worsening flips from $M_a$ to $M_b$ confined to $M$.  Finally, we say that $\cin$ is \emph{c-consistent} w.r.t $M$ if there is no $M_a \subseteq M$ s.t. $M_a \cseqt{\cin}{M} M_a$.  
\end{definition}

\noindent We can say that $\cseqt{\cin}{M}$ defines a form of confined reasoning on $\cin$ w.r.t. a $M \in \infm$. Proposition \ref{pro:09}, Corollary \ref{cor:09}, and Corollary \ref{cor:091} are straightforward consequences of Definitions \ref{def:06} and \ref{def:07} as well as Proposition \ref{pro:06}. They give a characterisation of reasoning about $\pci$-nets in terms of confined reasoning.  

\begin{propn} \label{pro:09} $M_a \cseqo{\cin} M_b$ iff $M_a \cseqt{\cin}{M} M_b$ for some $M \in \infm$
\end{propn}

\begin{corollary} \label{cor:09} If $\cin$ is satisfiable, then $M_a >_{\cin} M_b$  iff $M_a \cseqt{\cin}{M} M_b$ for some $M \in \infm$.
\end{corollary}

\begin{corollary} \label{cor:091} $\cin$ is satisfiable iff $\cin$ is \emph{c-consistent} w.r.t every $M \in \infm$.
\end{corollary}

Now usually one will only be interested in determining whether $M_a >_{\cin} M_b$ holds for some  $(M_a,M_b) \in U$  where $U$ is a finite subset of $\infm\times\infm$ (in particular, $|U| = 1$).  Such a $U$ is what we called an ``evaluation context'' in the introduction.   Hence one would ideally also like to know some (small) $M \in \infm$ s.t. $\cseqt{\cin}{M}$ \emph{captures} $\cseqo{\cin}$ for $U$, i.e. $M_a \cseqo{\cin} M_b$ iff $M_a \cseqt{\cin}{M} M_b$ for every $(M_a,M_b) \in U$.  

\begin{example} Consider again the $\pci$-net $\cin$ from Example \ref{ex:001}.  This $\pci$-net has the property that the dependencies between objects as given by the comparison expressions are ``acyclic''.  We don't develop this notion formally here, but it can be easily generalised from the notion of a $\sci$-net having an acyclic dependency graph as defined in Section \ref{sec::ci}.  This means that given an initial multi-set $M_a$, lets say $M_a = \{(a,3),(b,3)\}$, one can compute an upper bound on the number of instances of each object one will be able to add to the objects in $M_a$ via worsening flips.  Let $\#o$ denote this number for each $o \in O$.  Then
\begin{align*}
\#a &= 3,  \\
\#b &= 3 + (\#a *6) = 21,\\
\#c &= ( \#a *6) + (\#b *3) = 81, \\
\#d &=  (\#a * 6) + (\#b * 3) + (\#c * 3) = 324\\
\end{align*}
\noindent and therefore $\cseqt{\cin}{M}$, with $M = \{(a,3),(b,21),(c,81),(d,324)\}$, \emph{captures} $\cseqo{\cin}$ for $U = \{(M_a,M') \mid M' \in \infm\}$.  
\end{example}

\subsubsection{$\sci$-nets as a special case of $\pci$-nets}

We are now also in position to see in what sense $\sci$-nets are a special case of $\pci$-nets.  Translated into the context of $\pci$-nets a $\sci$-statement $c = S^+,S^-:S_1 \vartriangleright S_2 $ can be written as the $\pci$-statement $\hat{c} := P^+:C$ where
\begin{align*}
P^+ &:= P^+_{1} \cup P^+_{2} \cup P^+_{3}, \\
P^+_{1} &:= \{ \ce{s^+}{1}  \mid s^+ \in S^+ \}, \\
P^+_{2} &:= \{ \ce{s}{0}  \mid s^- \in (S^- \cup S_1 \cup S_2)   \}, \\
P^+_{3} &:= \{ \cl{s}{1}  \mid s \in (O \setminus (S^+ \cup S^- \cup S_1 \cup S_2)) \}, 
\end{align*}
\noindent and 
\begin{align*}
C &:= \{\nrp{s_1}{1}  \mid s_1 \in S_1 \} \vartriangleright \{\nrp{s_2}{1}  \mid s_2 \in S_2 \}.  
\end{align*}
\noindent If, given a $\sci$-net $\cin$, one constructs a $\pci$-net $\widehat{\cin} := \{\hat{c} \mid c \in \cin \}$ it is relatively easy to see that the CI-flips w.r.t. $\cin$ and $\hat{\cin}$ are, by construction, exactly the same.  Hence one also has that $\cseqt{\widehat{\cin}}{O}$ captures $\cseqo{\widehat{\cin}}$ for $U = (O \times O)$.   

\begin{example} The $\sci$-net from Example \ref{ex:01} as a $\pci$-net (assuming for simplicity that $O = \{a,b,c,d\}$) is the following:
\begin{align*}
\ce{a}{1},\sce{b,c,d}{0} : \nrp{d}{1} \vartriangleright \nrp{b}{1},\nrp{c}{1} ;\\
\ce{a}{1},\sce{b,c,d}{0} : \nrp{b}{1} \vartriangleright \nrp{c}{1} ;\\
\cl{a}{1},\sce{b,c}{0},\ce{d}{1} : \nrp{b}{1} \vartriangleright \nrp{c}{1}.
\end{align*}
\end{example}

\section{Confined reasoning about $\pci$-nets}
\label{sec:mci}

In this Section we develop the fundamentals for computational procedures for confined reasoning for $\pci$-nets and thus, via 
corollaries \ref{cor:09} and \ref{cor:091}, also for non-confined reasoning about $\pci$-nets.  More concretely, assume one has a $\pci$ net $\cin$, an ``evaluation context'' $U \subseteq (\infm \times \infm)$, as well as some $M \in \infm$ s.t. $\cseqt{\cin}{M}$ captures $\cseqo{\cin}$ for $U$.  Alternatively, one is just interested in $\cseqt{\cin}{M}$ without knowing for certain whether $\cseqt{\cin}{M}$ captures $\cseqo{\cin}$ for $U$.  E.g. one has some heuristic for determining an $M \in \infm$ s.t.  $\cseqt{\cin}{M}$ captures $\cseqo{\cin}$ for $U$ or one could compute $M$ in an iterative fashion.  The results in this Section then provide a basis for sound (and depending on whether $\cseqt{\cin}{M}$ captures $\cseqo{\cin}$ for $U$, also complete)  procedures  for determining   $\cseqo{\cin}$ for $U$.

Specifically, in this Section we first introduce a generalisation of $\sci$-nets for specifying and reasoning about preferences over multi-sets where the multiplicites of the items are bounded by some number.  We call these $\mci$-nets.  We then show how to translate confined reasoning about $\pci$-nets to reasoning about $\mci$-nets.  We also finally give an efficient reduction of reasoning for $\mci$-nets to reasoning about $\sci$-nets.

\subsection{$\mci$-nets}
\label{sec::def2}

$\mci$-nets consist of, as to be expected by now, a set of $\mci$-statements.  These are defined in a manner that follows closely the definition of $\sci$-statements.  

\begin{definition}[$\mci$-statements] \label{def::03}
Let $O$ be a finite set and $M$ a finite multiset on $O$. A $\mci$ statement on $M$ is an expression of the form
\begin{align*}
M^+,M^- : M_1 \vartriangleright M_2
\end{align*}
\noindent where $M^+ \subseteq M$, $M^- \subseteq (M \setminus M^+$), $M_1, M_2 \subseteq (M \setminus (M^+ \cup M^-))$, $M_1 \neq \emptyset$, $M_2 \neq \emptyset$, and $(M_1 \cap M_2) = \emptyset$.
\end{definition}

\noindent The constraints on $M^+ , M^- , M_1 , M_2$ in Definition \ref{def::03} amount to $M^+ , M^- , M_1 , M_2$ being disjoint sets when $M$ is a set.  Hence the definition in this case amounts to Definition \ref{def::scistmnt}.  The semantics of $\mci$-statements also is a direct generalisation of that of $\sci$-statements, although now defined for preference relations over $2^M$.

\begin{definition}[Semantics of $\mci$-statements] A preference relation over $2^M$ \label{def:44}  satisfies a $\mci$-statement $M^+,M^- : M_1 \vartriangleright M_2$ if for every $M' \subseteq (M  \setminus (M^+ \cup M^- \cup M_1 \cup M_2))$, we have $(M' \cup M^+ \cup M_1) > (M' \cup M^+ \cup M_2)$.
\end{definition}

\noindent Note that $M^+ \subseteq M$, $M^- \subseteq (M \setminus M^+)$, $M_1, M_2 \subseteq (M \setminus (M^+ \cup M^-))$, and $(M_1 \cap M_2) = \emptyset$ imply that $(M^+ \cup M^- \cup M_1 \cup M_2) \subseteq M$ and, hence, $M'$ in Definition \ref{def:44} is well defined. 

Now the notions of a $\mci$-net, a preference relation satisfying a $\mci$-net, a $\mci$-net being satisfiable, as well as the induced preference relation for a $\mci$-net (which for a $\mci$-net $\cin$ we again denote $>_{\mathcal{N}}$), are defined analogously as for $\sci$ and $\pci$ nets.  Together with the already referred to identity of Definition \ref{def::03} and Definition \ref{def::scistmnt} when $M$ is a set, this means, in particular, that $\mci$-nets are indeed a generalisation of $\sci$-nets.  It is also easy to show that a $\mci$-net on $M$ can express all monotonic preference relations on $2^M$. Example \ref{ex:001m} shows the result of encoding confined reasoning on the $\pci$-net of Example \ref{ex:001} w.r.t. $M = \{(A,6),(B,6),(C,6)\}$ as a $\mci$-net on $M$.  We develop the details of this encoding in the following sub-section.

The operational semantics for $\mci$-nets is also defined in analogous fashion as for $\sci$ and $\pci$ nets, via the notion of a worsening flip (for $\mci$-nets).  

\begin{definition}[Worsening flips for $\mci$-nets]
Let $\mathcal{N}$ be a $\mci$-net on $M$ and let $M_a,M_b \subseteq M$. Then $M_a \flip M_b$ is called a \emph{worsening flip} w.r.t. $\mathcal{N}$ if one of the two following conditions is satisfied:
\begin{itemize}
\item $M_a \supset M_b$ ($\M$ flip)
\item There is a $\mci$-statement $M^+,M^-:M_1 \vartriangleright M_2 \in \mathcal{N}$ and an $M' \subseteq (M \setminus (M^+ \cup M^- \cup M_1 \cup M_2))$ s.t. $M_a = (M' \cup M^+ \cup M_1)$ and  $M_b = (M' \cup M^+ \cup M_2)$ (CI flip).
\end{itemize}
A more operational characterisation equivalent to the latter condition is that if $\overline{M} = (M \setminus (M^+ \cup M^- \cup M_1 \cup M_2))$, then 
\begin{itemize}
\item $(M \setminus (M^- \cup M_2)) \supseteq M_a \supseteq (M_1 \cup M^+)$, 
\item $(M \setminus (M^- \cup M_1)) \supseteq M_b \supseteq (M_2 \cup M^+)$, and
\item $(\overline{M} \cap M_a) = (\overline{M} \cap M_b)$.
\end{itemize}
\end{definition}

\noindent We again denote there existing a sequence of worsening flips from $M_a$ to $M_b$ w.r.t. $\mathcal{N}$ as $M_a \cseqo{\cin} M_b$.

\begin{propn} \label{seq:m} Let $\mathcal{N}$ be a satisfiable $\mci$-net on $M$, and $M_a, M_b \subseteq M$. We have $M_a >_{\mathcal{N}} M_b$ if and only if $M_a \cseqo{\cin} M_b$.
\end{propn}

\subsection{Translating confined reasoning about $\pci$-nets to reasoning about $\mci$-nets}

Let $\cin$ be a $\pci$ net, $M \in \infm$.  We will in the following reduce confined reasoning w.r.t. $\cin$ and a $M \in \infm$ to reasoning about a $\mci$-net on $M$. Our reduction will be in two steps.  First we translate the $\pci$-net $\cin$ into a $\pci$-net $\cin'$ s.t. the CI flips $M_a \flip M_b$ w.r.t. $\cin'$ are exactly those CI flips $M_a \flip M_b$ w.r.t. $\cin$ s.t. $M_a,M_b \subseteq M$.
We will say that we are ``reinterpreting'' $\cin$ w.r.t. $M$.  In a second step it will then be easy to convert the $\pci$-net $\cin'$ into a $\mci$-net $\cin_M$ s.t. $M_a \cseqt{\cin'}{M} M_b$ iff $M_a \cseqo{\cin_M} M_b$. We will be informal in this section, since the translation is quite straightforward.

The basis of reinterpreting a $\pci$-net $\cin$ w.r.t. $M$ is to rewrite each $\pci$ statement $c = P^+:P_1 \vartriangleright P_2$  in $\cin$ into a $\pci$ statement $c' =  P^{+'}:P_1 \vartriangleright P_2$ s.t. $M' \in P^{+'}$ iff $M' \in P^+$, $(M' \cup M_{P_1}) \subseteq M$, and $(M' \cup M_{P_2}) \subseteq M$. We will say that $c$ is \emph{ meaningful w.r.t.} $M$ if such a $M'$ exists and call any such $M'$ a multi-set that \emph{makes $c$ meaningful w.r.t. $M$}.  Also, we will call such a $\pci$ statement $c'$ \emph{equivalent to $c$ for $M$}.   

So assume we have a $\pci$ statement $c = P^+ : P_1 \vartriangleright P_2$ that is meaningful w.r.t. $M$.  Then, since $c$ is also satisfiable note that the precondition and comparison expressions can be written in the form
\begin{align*}
P^+ &= \{\cg{o^+_i}{a_i^+}\}_{1 \leq i \leq p} \cup \{\cl{o^-_j}{a_j^-}\}_{1 \leq j \leq q}, \\
P_1 &= \{\nrp{o_k^1}{a_k^1}\}_{1 \leq k \leq r}, \\
P_2 &= \{\nrp{o_l^2}{a_l^2}\}_{1 \leq l \leq s} 
\end{align*}
\noindent where each $o \in O$ appears at most once in a sub-expression of the form $\cg{o^+_i}{a_i^+}$ and at most once in a sub-expression of the form $\cl{o^-_j}{a_j^-}$ in the precondition.

Let $O^*$ be defined as the set
\begin{align*}
\{o^+_i\}_{1 \leq i \leq p} \cup \{o^-_j\}_{1 \leq j \leq q} \cup \{o^1_k\}_{1 \leq k \leq r} \cup \{o^2_l\}_{1 \leq l \leq s}  
\end{align*}
\noindent We re-label the objects in $O^* \subseteq O$ to $\{o_1,\ldots,o_m\}$ ($m = p + q +r+s$). We can now rewrite the precondition and comparison expressions of $c$ as (we abuse the notation by using the same names as before)
\begin{align*}
P^+ &= \{\cg{o_h}{A_h^+}\}_{1 \leq h \leq m, A_h^+ > 0} \cup \{\cl{o_h}{A_h^-}\}_{1 \leq h \leq m},\\
P_1 &= \{\nrp{o_h}{A_h^1}\}_{1 \leq h \leq m,  A_h^1 > 0}, \\
P_2 &= \{\nrp{o_h}{A_h^2}\}_{1 \leq h \leq m ,  A_h^2 > 0} 
\end{align*}
\noindent where for each $1 \leq h \leq m$ we define
\begin{itemize}
\item  $A_h^+ := a_i^+$ if there is a $i \in  \{1,\ldots,p\}$ s.t. $o_h = o^+_i$, $A_h^+ := 0$ otherwise;
\item $A_h^- := a_j^-$ if there is a $j \in  \{1,\ldots,q\}$ s.t. $o_h = o^-_j$, $A_h^- := m_M(o_h)$ otherwise;
\item $A_h^1 := a_k^1$ if there is a $k \in  \{1,\ldots,r\}$ s.t. $o_h = o^1_k$, $A_h^1 := 0$ otherwise;
\item $A_h^2 := a_l^1$ if there is a $l \in  \{1,\ldots,s\}$ s.t. $o_h = o^2_l$, $A_h^2 := 0$ otherwise.
\end{itemize}
\noindent It should be clear that this rewriting produces a $\pci$ statement which has the same multisets which make it meaningful w.r.t. $M$ as the original $\pci$-statement.  We further simplify the precondition $P^+$ to 
\begin{align*}
\{\cg{o_h}{A_h^+}\}_{1 \leq h \leq m, A_h^+ > 0} \cup \{\cl{o_h}{B_h^-}\}_{1 \leq h \leq m}
\end{align*}
\noindent where $B_h^-$ is defined as
\begin{align*}
 max \{ I \mid A_h^+ \leq I \leq A_h^- \textit{ and } I + A_h^1 + A_h^2 \leq m_M(o_h)\}.
\end{align*}
\noindent It should still be relatively straightforward to see that this rewriting produces a $\pci$ statement $c'$ that is equivalent to the original $\pci$ statement $c$ for $M$. As a consequence, we have that for the $\pci$ net $\cin'$ obtained from $\cin$ by translating every $c \in \cin$ that is meaningful for $M$ in the manner just described, it holds that $M_a \cseqt{\cin'}{M} M_b$ iff $M_a \cseqt{\cin}{M} M_b$.

Now, for the second step of our reduction of confined reasoning for $\pci$-nets to reasoning on $\mci$-nets, consider again a $\pci$ statement $c' \in \cin'$ obtained from a $c \in \cin$ as described above. We now define 
\begin{align*}
M^+ &:= \{(o_h,A_h^+) \mid 1 \leq h \leq m, A_h^+ > 0\},\\
X_{o_h} &:= m_M(o_h)-B_h^--A_h^1-A_h^2, \\
M^- &:= \{(o_h,X_{o_h}) \mid 1 \leq h \leq m, X_{o_h} > 0)\}, \\
M^1 &:= \{(o_h,A_h^1) \mid 1 \leq h \leq m, A_h^1 > 0\},\\
M^2 &:= \{(o_h,A_h^2) \mid 1 \leq h \leq m, A_h^2 > 0\}.  
\end{align*}
Its then relatively straightforward to see, first of all, that $c'' = M^+,M^-: M_1 \vartriangleright M_2$ is a $\mci$ statement. Consider now the $\mci$ net $\cin_M$ consisting of such $\mci$ statements $c''$ for each $c' \in \cin'$.  We have that $M_a \flip M_b$ is a CI flip w.r.t. $c'$ in $\cin'$ iff it is a CI flip w.r.t. $c'' \in \cin_{M}$.  As a consequence, $M_a \cseqt{\cin'}{M} M_b$ iff $M_a \cseqo{\cin_M} M_b$; hence, also $M_a \cseqt{\cin}{M} M_b$ iff $M_a \cseqo{\cin_M} M_b$. 

\begin{example}\label{ex:001m}
The reinterpretation of the $\pci$ net from Example \ref{ex:001} w.r.t. $M = \{(A,6),(B,6),(C,6)\}$ is
\begin{align*}
\cl{a}{5},\scl{b,c,d}{0} : \nrp{a}{1} \vartriangleright \snrp{b,c,d}{6};\\
\cg{a}{1},\cl{b}{5},\scl{c,d}{3} : \nrp{b}{1} \vartriangleright \snrp{c,d}{3}; \\
\cg{a}{3},\cl{b}{2},\scl{c,d}{3} : \nrp{c}{1} \vartriangleright \nrp{d}{3}.
\end{align*}
\noindent The translation to a $\mci$-net is then
\begin{align*}
\{(a,1)\} \vartriangleright \{(b,6),(c,6),(d,6)\};\\
\{(a,1)\},\emptyset : \{(b,1)\} \vartriangleright \{(c,3),(d,3)\}; \\
\{(a,3)\},\{(b,4)\}:\{(c,3)\} \vartriangleright \{(d,3)\}.
\end{align*}
\end{example}

\subsection{Reduction of $\mci$-nets to $\sci$-nets}
\label{sec::reduc}

We now reduce preferences stated via $\mci$-nets to preferences stated using $\sci$-nets. More precisely, given a multiset $M$ and a $\mci$-net  $\mathcal{N}_M$ on $M$ we will define a $\sci$-net $\mathcal{N}_{S_M}$ for a set $S_M$ and a mapping of every $M' \subseteq M$ to a $\overline{M'} \subseteq S_M$ s.t.  $\mathcal{N}_M$ is satisfiable iff $\cin_{S_M}$ is satisfiable, and assuming $\cin_M$ is satisfiable, $M_a <_{\mathcal{N}_M} M_b$ iff $\overline{M_a} <_{\mathcal{N}_{S_M}} \overline{M_b}$.


%
We start by introducing some notation.  Given some $o \in O$ and $i,j$ s.t. $i,j \geq 1$ we define the \emph{forward-generated set of j indexed copies from i} of $o$ as
\begin{align*}
[o]_{i,j}^F := \{o_i,o_{i+1},\ldots,o_{i+(j-1)}\} 
\end{align*}   
\noindent and the \emph{backward-generated set of j indexed copies from i} of $o$ as
\begin{align*}
[o]_{i,j}^B := \{o_i,o_{i-1},\ldots,o_{i-(j-1)}\}. 
\end{align*}
\noindent If $j = 0$, we define
\begin{align*}
[o]_{i,j}^F = [o]_{i,j}^B := \emptyset.  
\end{align*}
\noindent Then 
\begin{align*}
S_M := \bigcup \{[o]_{1,m_M(o)}^F \mid o \in O \}
\end{align*}
\noindent We call $[o]_{1,m_M(o)}^F = [o]_{m_M(o),m_M(o)}^B$ for $o \in O$ the \emph{set of indexed copies of $o$ in $S_M$}.

For some $M' \subseteq M$, $ss_{S_M}(M')$ includes all sets which, for each $o \in O$, have the same \emph{number} of elements from the set of indexed copies of $o$ in $S_M$ as instances of $o$ there are in $M'$.
Formally, we define $ss_{S_M}(M')$ to be the set
\begin{align*}
\{S \subseteq S_M \mid |S \cap [o]_{1,m_M(o)}^F| = m_{M'}(o) \textit{ for every } o \in O\}.
\end{align*} 
\noindent Clearly, in particular $ss_{S_M}(M) = \set{S_M}$.

We will also (partially) order the sets in $S_M$ via the order $>_i$ defined as the transitive closure of the binary relation
\begin{align*}
\{(S_1, S_2) \mid (S_1 &\cup S_2) \setminus (S_1 \cap S_2) =\{o_i,o_j\} \textit{ s.t. }\\ 
&o \in O, o_i \in S_1, o_j \in S_2, \textit{and } j = i+1\}.  
\end{align*}
\noindent Crucial for our purposes is that there is a unique maximal element
\begin{align*}
max_{>_i}(ss_{S_M}(M'))=\bigcup \{[o]_{1,m_{M'}(o)}^F \mid o \in O'\}
\end{align*}
\noindent w.r.t. $>_i$ within $ss_{S_M}(M')$ for every $M' \subseteq M$\footnote{There is also a unique minimal element
\begin{align*}
\bigcup \{[o]_{m_M(o),m_{M'}(o)}^B \mid o \in O'\}
\end{align*}
w.r.t. $>_i$ within $ss_{S_M}(M')$ for every $M' \subseteq M$.
}. We denote this maximal element as $\overline{M'}$.  Note that in particular $\overline{M} = S_M$.

Next we proceed to define for a $\mci$-statement $c \in \mathcal{N}_M$ the corresponding $\sci$ statement $\overline{c} \in \mathcal{N}_{S_M}$.  Assume $c$ is of the form $M^+,M^- : M_1 \vartriangleright M_2$. 
Then 
\begin{align*}
\overline{c} := \widehat{M^+},\widehat{M^-} : \widehat{M_1} \vartriangleright \widehat{M_2}
\end{align*}
\noindent where 
\begin{align*}
 \widehat{M^+} &:= \bigcup \{[o]_{1,m^+(o)}^F \mid o \in O \} \\
\widehat{M^-} &:=  \bigcup \{[o]_{m_M(o),m^-(o)}^B \mid o \in O \} \\
\widehat{M_1} &:= \bigcup \{[o]_{m_M(o) - m^-(o),m_1(o)}^B \mid o \in O \}  \\
\widehat{M_2} &:= \bigcup \{[o]_{m^+(o)+1,m_2(o)}^F \mid o \in O \}.
\end{align*}
\noindent We denote the set of $\sci$ statements corresponding to the $c \in \mathcal{N}_M$ as $\overline{cs}$.  Apart from having a statement $\overline{c}$ as defined above for each $c \in \mathcal{N}_M$, $\mathcal{N}_{S_M}$ also contains the set of $\sci$-statements
\begin{align*}
cs_i := \{ o_i \vartriangleright o_j \mid o \in O, 1 \leq i < m_M(o), j = i+1 \}
\end{align*} 
Now Corollaries \ref{corra} and \ref{corrb}
basically follows from Proposition \ref{prop:red}  which, in turn, follows mainly from lemmas \ref{lem:red1} and \ref{lem:red2}.

\begin{lemma} \label{lem:red1} If $M_a \flip M_b$  is a CI-flip w.r.t $c \in \mathcal{N}_M$ then there is a $S_{M_a} \in ss_{S_M}(M_a)$ s.t. $S_{M_a} \flip \overline{M_b}$ is a CI-flip w.r.t. $\overline{c} \in \mathcal{N}_{S_M}$.  Also, if there are $S_{M_a} \in ss_{S_M}(M_a)$, $S_{M_b} \in ss_{S_M}(M_b)$ s.t. $S_{M_a} \flip S_{M_b}$ is a CI-flip w.r.t. $\overline{c} \in \mathcal{N}_{S_M}$, then $M_a \flip M_b$  is a CI-flip w.r.t $c \in \mathcal{N}_M$.
\end{lemma}   

\begin{proof} (sketch) Let $c = M^+,M^-:M_1 \vartriangleright M_2$.  The lemma follows from the fact that the set of all CI-flips ``induced'' by $\overline{c} \in \mathcal{N}_{S_M}$ are of the form 
\begin{align*}
(S' \cup \widehat{M^+} \cup \widehat{M_1}) \flip (S' \cup \widehat{M^+} \cup \widehat{M_2})
\end{align*}
\noindent with $\widehat{M^+}, \widehat{M_1}, \widehat{M^+}, \widehat{M_2}$ defined as above, i.e 
\begin{align*}
\widehat{M^+} &= \bigcup \{[o]_{1,m^+(o)}^F \mid o \in O^+ \} \\
\widehat{M_1} &= \bigcup \{[o]_{m_M(o) - m^-(o),m_1(o)}^B \mid o \in O_1 \} \\
\widehat{M^+} \cup \widehat{M_2} &= \bigcup \{[o]_{1, m^+(o)+m_2(o)}^F \mid o \in O^+ \cup O_2  \},
 \end{align*}
\noindent and
\begin{align*}
S' &\subseteq \bigcup \{[o]_{X_o,Y_o}^F \mid o \in O  \}
\end{align*}
\noindent where 
\begin{align*}
X_o &= m^+(o)+m_2(o)+1 \\
Y_o &= m_{M}(o) - (m^+(o) + m^-(o)+m_1(o) + m_2(o)).
\end{align*}
\noindent for each $o \in O$.

Now, let $M_a \flip M_b$  be a CI-flip w.r.t $c \in \mathcal{N}_M$, i.e. there is a $M' \subseteq (M  \setminus (M^+ \cup M^- \cup M_1 \cup M_2))$, s.t. $M_a = (M' \cup M^+ \cup M_1)$ and $M_b =  (M' \cup M^+ \cup M_2)$.  Then $m_{M'}(o) \leq (m_{M}(o) - (m^+(o) + m^-(o)+m_1(o) + m_2(o)))$ for every $o \in O$ and it is relatively straightforward to see that for 
\begin{align*}
S' = \bigcup \{[o]_{X_o,m_{M'}(o)}^F \mid o \in O  \}
\end{align*}
\noindent  it is the case that $(S' \cup \widehat{M^+} \cup \widehat{M_2}) = \overline{M_b}$.  Also, it is straight-forward to verify that $S_{M_a} \in ss_{S_M}(M_a)$  
for $S_{M_a} := S' \cup \widehat{M^+} \cup \widehat{M_1}$.

On the other hand, assume now that there are $S_{M_a} \in ss_{S_M}(M_a)$, $S_{M_b} \in ss_{S_M}(M_b)$ s.t. $S_{M_a} \flip S_{M_b}$ is a CI-flip w.r.t. $\overline{c} \in \mathcal{N}_{S_M}$.  This means that $S_{M_a} = (S' \cup \widehat{M^+} \cup \widehat{M_1})$ and $S_{M_b} = (S' \cup \widehat{M^+} \cup \widehat{M_1})$ for some $S'$ as indicated above.  Now consider $M'$ s.t. $m_{M'}(o) = |S' \cap [o]_{1,m_M(o)}^F|$.  Then it is straightforward to see that $M_a = (M' \cup M^+ \cup M_1)$, $M_b = (M' \cup M^+ \cup M_2)$ and $M_a \flip M_b$ is a CI-flip w.r.t $c \in \mathcal{N}_M$.
\end{proof}

\begin{lemma} \label{lem:red2} If $S >_{i} S'$ then there is a sequence of $cs_i$ flips from $S$ to $S'$  w.r.t. $\mathcal{N}_{S_M}$.
\end{lemma}

\begin{proof}(sketch) This lemma follows from the definition of $>_i$, $cs_i$, and the definition of CI flips for $\sci$ statements (Definition \ref{def:opsci}). 
\end{proof}

\begin{propn} \label{prop:red} Let $M_a,M_b \subseteq M$.  If $M_a \cseqo{\cin_M} M_b$, then $\overline{{M_a}} \cseqo{\cin_{S_M}} \overline{{M_b}}$.  Also, if  $S_{M_a} \cseqo{\cin_{S_M}} S_{M_b}$ for some $S_{M_a} \in ss_{S_M}(M_a)$, $S_{M_b} \in ss_{S_M}(M_b)$ then $M_a \cseqo{\cin_M} M_b$. 
\end{propn}

\begin{proof} 

We start by proving by induction on $k \geq 0$, that if there exists a sequence $M_a,\ldots,M_b$ w.r.t. $\mathcal{N}_M$ with $k$ CI flips, then there is a sequence $\overline{M_a},\ldots,\overline{M_b}$ w.r.t. $\mathcal{N}_{S_M}$ with $k$ $\overline{cs}$ flips.  The base case ($k = 0$) follows from the fact that if $M_a \supset M_b$ then $\overline{M_a} \supset \overline{M_b}$ and hence there is a sequence $\overline{M_a},\ldots,\overline{M_b}$ consisting only of $\M$ flips w.r.t. $\cin_{S_M}$.

For the inductive case assume that there exists a sequence $M_a,\ldots, M_b$ w.r.t $\cin_M$ with $k+1 \geq 1$ CI flips.  Consider the last $M_c,M_d$ in the sequence s.t. $M_c \flip M_d$ is a CI flip.  By inductive hypothesis then there is a sequence of flips $\overline{M_a},\ldots,\overline{M_c}$ w.r.t. $\cin_{S_M}$ with $k$ $\overline{cs}$ flips.  By Lemma \ref{lem:red1} there exists a $S_{M_c} \in ss_{S_M}(M_c)$ s.t. $S_{M_c} \flip \overline{M_d}$ is a CI-flip w.r.t. $\overline{c} \in \mathcal{N}_{S_M}$.  Hence $\overline{M_a},\ldots,\overline{M_c},\ldots,S_{M_c},\overline{M_d},\ldots,\overline{M_b}$ is a sequence w.r.t. $\cin_{S_M}$ with $k+1$ $\overline{cs}$ flips. Here $\overline{M_c},\ldots,S_{M_c}$ is a sequence of (possibly 0) $csi$ flips (that such a sequence exists follows from Lemma \ref{lem:red2}) and $\overline{M_d},\ldots,\overline{M_b}$ is a sequence of (possibly 0) $\M$ flips.

We now prove by induction on $k \geq 0$, that if there exists a sequence $S_{M_a}, \ldots, S_{M_b}$ w.r.t. $\mathcal{N}_{S_M}$ with $k$ $\overline{cs}$ flips for some $S_{M_a} \in ss_{S_M}(M_a)$ and $S_{M_b} \in ss_{S_M}(M_b)$, then there exists a sequence $M_a,\ldots,M_b$ w.r.t. $\mathcal{N}_M$ with $k$ CI flips. The base case ($k=0$) follows from the fact that any sequence leading from some $S_{M_a} \in ss_{S_M}(M_a)$ to some $S_{M_b} \in ss_{S_M}(M_b)$ with $M_a \neq M_b$ must have at least one $\M$ flip, i.e. $M_a \supset M_b$, and hence $M_a,M_b$ is a sequence consisting only of one $\M$ flip  w.r.t. $\cin_M$. 

For the inductive case assume that there exists a sequence  $S_{M_a}, \ldots, S_{M_b}$ w.r.t. $\mathcal{N}_{S_M}$ with $k+1 \geq 1$ $\overline{cs}$ flips for some $S_{M_a} \in ss_{S_M}(M_a)$ and $S_{M_b} \in ss_{S_M}(M_b)$.  Consider the last $\overline{cs}$ flip $S_{M_c} \flip S_{M_d}$ in the sequence, with $S_{M_c} \in ss_{S_M}(M_c)$, and $S_{M_d} \in ss_{S_M}(M_d)$ for $M_c,M_b \subseteq M$.  By inductive hypothesis, then there is a sequence $M_a,\ldots,M_c$ w.r.t. $\cin_M$ with $k$ CI flips.  By Lemma \ref{lem:red1} also $M_c \flip M_d$ is a CI flip w.r.t. $\cin_M$.  Now, by the same argument as for the base case, if $M_d \neq M_b$, $M_d \supset M_b$ must be the case. Hence  $M_a,\ldots,M_c,M_d,\ldots,M_b$ is a sequence w.r.t. $\cin_M$ with $k$ CI flips where $M_d,\ldots,M_b$ is a sequence of (possibly 0) $\M$ flips.  %
\end{proof}

\begin{corollary} \label{corra} $\cin_M$ is satisfiable iff $\cin_{S_M}$ is satisfiable. 
\end{corollary}

\begin{corollary} \label{corrb} Let $\cin_M$ be satisfiable.  Then $M_a <_{\cin_M} M_b$ iff 
$\overline{{M_a}} <_{\cin_{S_M}} \overline{{M_b}}$.
\end{corollary}

\begin{example} \label{ex:76}
\noindent The following is the $\sci$-net corresponding to the $\mci$-net from Example \ref{ex:001m}.  
\begin{align}
&\{a_6\} \vartriangleright \{b_1,\ldots,b_6,c_1,\ldots,c_6,d_1,\ldots,d_6\}; \label{761}\\
&\{a_1\},\emptyset : \{b_6\} \vartriangleright \{c_1,c_2,c_3,d_1,d_2,d_3\}; \label{762}\\
&\{a_1,a_2,a_3\},\{b_3,b_4,b_5,b_6\}: \{c_4,c_5,c_6\} \vartriangleright \{d_1,d_2,d_3\}; \label{763}\\
&\{a_i \vartriangleright a_{i+1} \mid 1 \leq i \leq 5 \}; \label{764}\\
&\{b_i \vartriangleright b_{i+1} \mid 1 \leq i \leq 5 \}; \label{765}\\
&\{c_i \vartriangleright c_{i+1} \mid 1 \leq i \leq 5 \}; \label{766}\\
&\{d_i \vartriangleright d_{i+1} \mid 1 \leq i \leq 5 \} \label{767}
\end{align}
\noindent Here  
\begin{align*}
S_M = \{a_1,\ldots,a_6,b_1,\ldots,b_6,c_1,\ldots,c_6\}.
\end{align*}
\noindent Note that the $\sci$-net has an acyclic depdency graph.  In fact, it is easy to see that the dependency graph of any $\sci$-net resulting from the reduction of reasoning about an $\mci$-net which encodes confined reasoning w.r.t. the $\pci$-net in Example \ref{ex:001} and \emph{any} $M \in \infm$ will be acyclic.  This is an indirect way of showing (via Corollary \ref{cor:091} and the result from \citeauthor{BouveretEL09} mentioned at the end of Section \ref{sec::ci}) that the $\pci$-net in Example \ref{ex:001} is satisfiable.    
Next we give the derivation w.r.t. the $\sci$-net above corresponding to that of Example \ref{ex:0001A}.  
\begin{align*}
&\{a_1,a_2,a_3,b_1,b_2,b_3\}\\
& \ldots (CI,\ref{765}) \\
&\{a_1,a_2,a_3,b_1,b_2,b_6\} \\
&\flip (CI,\ref{762}) \\
&\{a_1,a_2,a_3,b_1,b_2,c_1,c_2,c_3,d_1,d_2,d_3\} \\
&\ldots (CI,\ref{766}-\ref{767}) \\
&\{a_1,a_2,a_3,b_1,b_2,c_4,c_5,c_6,d_4,d_5,d_6\} \\
&\flip (CI,\ref{763}) \\
&\{a_1,a_2,a_3,b_1,b_2,d_1,d_2,d_3,d_4,d_5,d_6\} \\
& \flip  (\M) \\
&\{a_1,a_2,a_3,b_1,b_2,d_1,d_2,d_3,d_4,d_5\} 
\end{align*}
\end{example}

\section{Encoding preferences in evidence aggregation using $\pci$-nets}
\label{sec:app}

In this Section we present a possible use case of $\pci$-nets.  Specifically, in the context of the relatively recent argument-based system for aggregating evidence about treatments resulting from clinical trials presented in \cite{HunterW12,HunterW15}.  We refer to the cited works for an in depth description of the system.  See also \cite{Williams2015290} for results on the use of the system for aggregating evidence from studies about lung chemo-radiotherapy.

\begin{table}[t]
\centering
\resizebox{0.5\textheight}{!}{
\begin{minipage}{\textwidth}

    \begin{tabular}{ | l | l | l | l | l | l | l | l |}
    \hline
    ID & Left & Right & Outcome  & Outcome  & Net  & Sig & Type\\
       &      &       & indicator & value & outcome &  & \\ \hline \hline
    $e_{01}$ & PG & BB & change in IOP & -2.32  & $>$ & no & MA  \\ \hline
    $e_{02}$ & PG & BB & acceptable IOP & 1.54 & $>$ & yes & MA  \\ \hline
    $e_{03}$ & PG & BB & respiratory prob & 0.9  & $>$ & yes & MA  \\ \hline
    $e_{04}$ & PG & BB & respiratory prob & 0.85  & $>$ & yes & MA  \\ \hline
    $e_{05}$ & PG & BB & cardio prob & 0.82  & $>$ & no & MA  \\ \hline    
    $e_{06}$ & PG & BB & hyperaemia & 0.61  & $<$ & yes & MA  \\ \hline
    $e_{07}$ & PG & BB & drowsiness & 0.58  & $<$ & yes & MA  \\ \hline
   $e_{08}$ & PG & BB & drowsiness & 0.71  & $<$ & yes & MA  \\ \hline
   $e_{09}$ & PG & BB & drowsiness & 0.62  & $<$ & yes & MA  \\ \hline
    \end{tabular}

\end{minipage}}
\caption{Normalised results of several meta-analysis studies comparing prostaglandin analogue (PG) and beta-blockers (BB) for patients with raised intraocular pressure.}
\label{tb:2}
\end{table}

In the already mentioned system evidence resulting from clinical trials is initially collected in the form of evidence tables of which Table \ref{tb:2} could be an extract (see Table 3 in \cite{HunterW12} for a larger real-world example on which our example is based).  This table summarises possible results stemming from meta-analysis (therefore the label ``$MA$'' in the column ``Type'') comparing two treatments-, ``$PG$'' for ``prostaglandin analgoue'' and ``$BB$'' for ``beta-blockers'', for patients who have raised intraocular pressure (i.e. raised pressure in the eye) and are therefore at risk of glaucoma with resulting irreversible damage to the optic nerve and retina.  Each evidence item has an associated ``ID'' indicated by the entry in the column ``ID''.   

The meta-analysis compare the treatments w.r.t. different outcome indicators (see column ``Outcome indicator'').  The results of the studies are stored in the column labelled ``Outcome value''.    For simplicity, the outcome values in Table \ref{tb:2} have already been normalised so that the values are desirable.  This means that the outcome value in each row express the degree to which the treatment which has fared better in the study corresponding to the row, has indeed done better. The entry under the column ``Net outcome'' indicates which of the treatments fared better, with ``$>$'' meaning that there were better results for prostaglandin analogue and ``$<$'' indicating that the results speak for beta-blockers.

For respiratory problems (``respiratory prob''), cardiological problems (``cardio prob''), hyperaemia (redness of eyes), and drowsiness the outcome values express the relative risk of suffering the considered problem.  E.g. for cardiological problems (evidence item $e_{05}$), this means that 82 people suffered respiratory problems when taking prostaglandin-analogue for every 100 persons suffering respiratory problems when taking beta-blockers.  ``change in IOP'' and ``acceptable IOP'' are outcome indicators referring to raised intraocular pressure, with negative outcome values being good for the first and outcome values greater than 1 being good for the second of these outcome indicators.  

Now, given evidence such as is summarised in Table \ref{tb:2}, the question is whether prostaglandin analogue or beta-blockers are better to treat glaucoma in patients who have raised intraocular pressure. The obvious problem being that there is evidence in favour of prostaglandin analogue (namely, evidence items $e_{01}-e_{05}$), but there is also evidence in favour of beta-blockers ($e_{06}-e_{09}$).  

A first step towards a solution of this problem is to help out the decision process by determining what sets of evidence items that can be used to argue in favour of the treatments are better in terms of preferences over the outcomes (``Outcome indicator'') and magnitudes (``Outcome value'') of the evidence items in such sets.  Since we consider that the outcome values are normalised as detailed above, we call such ``outcome indicator-value pairs'' ``\emph{benefits}''.  
More to the point, since for methodological reasons (mainly, to avoid bias and for purposes of reuse), preferences need to be determined independently of the available evidence, the preference relation needs to be in terms of \emph{possible} sets of benefits, i.e. all possible sets of pairs of (normalised) outcome indicator-value pairs.  

This process of ordering the evidence in terms of a preference relation on the ``benefit sets'' allows for the incorporation of a ``personalised'' dimension in the decision process, i.e. of considerations which have to do with, for example, a specific patient or the experience of the medical professional.  Other more ``objective'' elements (like the statistical significance of the results obtained via the studies in Table \ref{tb:2} - see the entries under the column ``Sig'') can be incorporated in further stages of the decision process as outlined in \cite{HunterW12,HunterW15}.

In \cite{HunterW12,HunterW15} \citeauthor{HunterW12} only consider the incoporation of preferences between \emph{sets} of benefits in their system and for specifying such preference relations $\sci$-nets are a natural choice (albeit, not one considered in \cite{HunterW12,HunterW15}).  First and most obviously because sets of items need to be compared and secondly, because preferences over sets of benefits can also be considered to be monotonic, i.e. having more evidence in favour of a treatment being beneficial can usually be considered to be better than having less evidence.    

Now, when one considers that there may be more than one evidence item expressing the same benefit, it is clear that at least theoretically the preference relation should be over multi-sets of benefits rather than sets of benefits.  But also from a practical perspective it may often be the case that one should at least allow for both preferences over multi-sets as well as sets of benefits to be specified.
Enabling this option becomes especially relevant when, as will often be the case, one introduces some abstraction over the outcome indicators and values appearing in the benefit sets (thus, simplifying the task of specifying the preference relation).

Example \ref{ex:sp1} illustrates the use of enabling encoding multi-set preferences by giving a specification of a preference relation over multi-sets of benefits such as those appearing in Table \ref{tb:2} but where we introduce a natural abstraction over the outcome indicators and values.  First of all we consider both ``change in IOP'' and ``acceptable IOP'' as part of the ``significant outcomes'' which we denote ``SO''. Secondly, we partition the outcome indicators into ``s'', ``m'', and ``l'' standing for a ``small'', ``medium'', and ``large'' improvement respectively.  We don't go into the details here, but Table \ref{tb:18} shows a possible result of applying this abstraction to the otucome-indicators and values appearing in Table \ref{tb:2}.

\begin{table}[t]
\centering
\resizebox{0.5\textheight}{!}{
\begin{minipage}{\textwidth}

    \begin{tabular}{ | l | l | l | l | l | l | l | l |}
    \hline
    ID & Left & Right & Outcome  & Outcome  & Net  & Sig & Type\\
       &      &       & indicator & value & outcome &  & \\ \hline \hline
    $e_{01}$ & PG & BB & SO & m  & $>$ & no & MA  \\ \hline
    $e_{02}$ & PG & BB & SO & s  & $>$ & yes & MA  \\ \hline
    $e_{03}$ & PG & BB & respiratory prob & s  & $>$ & yes & MA  \\ \hline
    $e_{04}$ & PG & BB & respiratory prob & s  & $>$ & yes & MA  \\ \hline
    $e_{05}$ & PG & BB & cardio prob & s  & $>$ & no & MA  \\ \hline    
    $e_{06}$ & PG & BB & hyperaemia & m  & $<$ & yes & MA  \\ \hline
    $e_{07}$ & PG & BB & drowsiness & m  & $<$ & yes & MA  \\ \hline
   $e_{08}$ & PG & BB & drowsiness & m  & $<$ & yes & MA  \\ \hline
   $e_{09}$ & PG & BB & drowsiness & m  & $<$ & yes & MA  \\ \hline
    \end{tabular}

\end{minipage}}
\caption{Results of meta-analysis comparing prostaglandin analogue (PG) and beta-blockers (BB) with abstractions over outcome indicators and values.}
\label{tb:18}
\end{table}

\begin{example} \label{ex:sp1} The following illustrates the use of $\pci$-nets for the specification of a preference relation over sets of benefits such as appear in Table \ref{tb:18}.  For simplicity we use the shorthands $C$, $D$, $H$, $R$ for $(\textit{cardio prob},s)$, $(\textit{drowsiness}, m)$, $(\textit{hyperaemia},m)$, and $(\textit{respiratory prob},s)$ respectively, while $Sm := (SO,m)$ and $Ss := (SO,s)$.     
\begin{align}
\nrp{Sm}{1} \vartriangleright \snrs{C,D,R,Ss}{} , \nrp{H}{1}; \label{rx11}  \\
\snrp{C,R}{1} \vartriangleright \nrp{D}{1}; \label{rx12}  \\
\ce{C}{0}: \nrp{H}{1} \vartriangleright \nrs{D},\snrp{R,Ss}{1};\label{rx13}  \\
\ce{R}{0}: \nrp{H}{1} \vartriangleright \nrs{D},\snrp{C,Ss}{1}; \label{rx14} \\
\ce{C}{0}: \nrp{D}{2} \vartriangleright \nrp{R}{1};\label{rx15}  \\
\ce{R}{0}: \nrp{D}{2} \vartriangleright \nrp{C}{1};  \label{rx16}\\
\ce{Sm}{0}: \nrp{Ss}{1} \vartriangleright \snrp{C,D,R}{1}; \label{rx17} \\
\cg{Sm}{1}: \snrp{C,R}{1} \vartriangleright \nrp{Ss}{1} \label{rx18}  
\end{align}
\end{example}

In Example \ref{ex:sp1}, the $\pci$-statement \ref{rx11} basically says that evidence showing a medium improvement for any of the significant outcomes is preferred over any number of evidence regarding the side-effects cardio problems, respiratory problems, and drowsiness, as well as evidence showing a small improvement regarding the significant outcomes.  Also having more  evidence (exactly one more piece of evidence) for a medium improvement for the significant outcomes is preferred to having more evidence for a modest improvement w.r.t. hyperaemia ($H$).  

The $\pci$-statement \ref{rx12} states that having more evidence for a small improvement for both cardio and respiratory problems is preferred to having more evidence for a modest improvement regarding drowsiness.  The $\pci$-statements \ref{rx13} to \ref{rx16} state preferences for the scenario where one does not have any evidence for some improvement in regards to one of cardio or respiratory problems.  In this situation, for example having evidence for a modest improvement in hyperaemia is preferred to having evidence for a small improvement in only one of cardio or respiratory problems.  The same holds for drowsiness ($\pci$-statements \ref{rx15} and \ref{rx16}) although the standards here are set a big higher; one needs to have a difference in two studies showing a modest improvement in drowsiness.  

The $\pci$-statement \ref{rx17} states that if one does not have any evidence for a modest improvement in the significant outcomes, then evidence for even a small improvement for any of the significant outcomes is preferred to evidence showing an improvement in cardio problems, respiratory problems, and drowsiness.  On the other hand, if one already has some evidence for a medium improvement in the significant outcomes, then also having more evidence for a small improvement in both cardio and respiratory problems is preferred to having more evidence for a small improvement in the significant outcomes.  

Example \ref{ex:sp2} now shows the encoding of confined reasoning for the $\pci$-net of Example \ref{ex:sp1} w.r.t. the multiset
\begin{align*}
M = \{(C,1),(D,3),(H,1),(R,2),(Sm,1),(Ss,1)\},
\end{align*}
\noindent i.e. the multiset containing all benefits ocurring in Table \ref{tb:18}. Example \ref{ex:sp3} shows the reduction of the $\mci$-net from Example \ref{ex:sp2} to a $\sci$-net.  

\begin{example}\label{ex:sp2}
The following is the encoding of confined reasoning for the $\pci$-net of Example \ref{ex:sp1} w.r.t. the multiset
\begin{align*}
M = \{(C,1),(D,3),(H,1),(R,2),(Sm,1),(Ss,1)\}
\end{align*}
\noindent as detailed in Section \ref{sec:mci}.  For the encoding we interpret $\nrs{o}$ as the maximum number of ocurrences of $o$ in $M$.
\begin{align}
\{(Sm,1)\} \vartriangleright \{(C,1),(D,3),(H,1),(R,2),(Ss,1)\}; \label{r1} \\
\{(C,1),(R,1)\} \vartriangleright \{(D,1)\}; \label{r2} \\
\emptyset, \{(C,1)\}: \{(H,1)\} \vartriangleright \{(D,3),(R,1),(Ss,1)\}; \label{r3} \\
\emptyset, \{(R,2)\}: \{(H,1)\}  \vartriangleright \{(D,3),(C,1),(Ss,1)\}; \label{r4}\\
\emptyset, \{(C,1)\}: \{(D,2)\} \vartriangleright \{(R,1)\}; \label{r5}\\
\emptyset, \{(R,2)\}: \{(D,2)\} \vartriangleright \{(C,1)\}; \label{r6}\\
\emptyset,\{(Sm,1)\}: \{(Ss,1)\} \vartriangleright \{(C,1),(D,1),(R,1)\}; \label{r7} \\
\{(Sm,1)\},\emptyset: \{(C,1),(R,1)\} \vartriangleright \{(Ss,1)\} \label{r8} 
\end{align}
\end{example}

\begin{example}\label{ex:sp3} Next we present the translation of the $\mci$-net from Example \ref{ex:sp2} to a $\sci$-net.
\begin{align*}
\{Sm_1\} \vartriangleright \{C_1,D_1,D_2,D_3,H_1,R_1,R_2,Ss_1\}; \\
\{C_1,R_2\} \vartriangleright \{D_1\}; \\
\emptyset, \{C_1\}: \{H_1\} \vartriangleright \{D_1,D_2,D_3,R_1,Ss_1\}; \\
\emptyset, \{R_1,R_2\}: \{H_1\}  \vartriangleright \{D_1,D_2,D_3,C_1,Ss_1\}; \\
\emptyset, \{C_1\}: \{D_2,D_3\} \vartriangleright \{R_1\}; \\
\emptyset, \{R_1,R_2\}: \{D_2,D_3\} \vartriangleright \{C_1\}; \\
\emptyset,\{Sm_1\}: \{Ss_1\} \vartriangleright \{C_1,D_1,R_1\}; \\
\{Sm_1\},\emptyset: \{C_1,R_2\} \vartriangleright \{Ss_1\}; \\
\{D_1\} \vartriangleright \{D_2\}; \\
\{D_2\} \vartriangleright \{D_3\}; \\
\{R_1\} \vartriangleright \{R_2\} 
\end{align*}
\end{example}

\noindent Figure \ref{fig:11} shows the preference relation induced by the $\mci$-net in Example \ref{ex:sp2}, but considering only sets of benefits which all result from the \emph{same} treatment according to the evidence in Table \ref{tb:18}.

\begin{figure}[t!]

\centering
\resizebox{0.7\textheight}{!}{\begin{minipage}{\textwidth}
\begin{tikzpicture}[main node/.style={fill=none,font=\scriptsize}]

\node[main node] (crrss) at (-7,0) {$\{C,R,R,Sm,Ss\}$};
\node[main node] (crrs) at (-11,-1.5) {$\{C,R,R,Sm\}$};

\node[main node] (crrs2) at (-9,-1.5) {$\{C,R,R,Ss\}$};
\node[main node] (crss) at (-7,-1.5) {$\{C,R,Sm,Ss\}$};

\node[main node] (rrss) at (-5,-1.5) {$\{R,R,Sm,Ss\}$};

\node[main node] (crr) at (-12.4,-3) {$\{C,R,R\}$};
\node[main node] (crs) at (-10.9,-3) {$\{C,R,Sm\}$};
\node[main node] (crs2) at (-9.5,-3) {$\{C,R,Ss\}$};
\node[main node] (css) at (-7.7,-3) {$\{C,Sm,Ss\}$};
\node[main node] (rrs) at (-6.2,-3) {$\{R,R,Sm\}$};
\node[main node] (rrs2) at (-4.7,-3) {$\{R,R,Ss\}$};
\node[main node] (rss) at (-3,-3) {$\{R,Sm,Ss\}$};

\node[main node] (cr) at (-12,-4.5) {$\{C,R\}$};
\node[main node] (cs) at (-10.5,-4.5) {$\{C,Sm\}$};
\node[main node] (cs2) at (-9,-4.5) {$\{C,Ss\}$};
\node[main node] (rr) at (-7.5,-4.5) {$\{R,R\}$};
\node[main node] (rs) at (-6,-4.5) {$\{R,Sm\}$};
\node[main node] (rs2) at (-4.5,-4.5) {$\{R,Ss\}$};
\node[main node] (ss) at (-3,-4.5) {$\{Sm,Ss\}$};

\node[main node] (c) at (-10.5,-6) {$\{C\}$};
\node[main node] (r) at (-8.5,-6) {$\{R\}$};
\node[main node] (s) at (-6.5,-6) {$\{Sm\}$};
\node[main node] (s2) at (-4.5,-6) {$\{Ss\}$};

\node[main node] (dddh) at (-0.5,-0.5) {$\{D,D,D,H\}$};

\node[main node] (ddd) at (-1.5,-2) {$\{D,D,D\}$};
\node[main node] (ddh) at (0.5,-2) {$\{D,D,H\}$};

\node[main node] (dd) at (-1.5,-3.5) {$\{D,D\}$};
\node[main node] (dh) at (0.5,-3.5) {$\{D,H\}$};

\node[main node] (d) at (-1.5,-5) {$\{D\}$};
\node[main node] (h) at (0.5,-5) {$\{H\}$};

\path

(s) edge[->,dashed,out=300,in=280] node{\bf\scriptsize \ref{r1}}(dddh)
(s) edge[->,dashed,out=145,in=280] node{\bf\scriptsize \ref{r1}}(crrs2)

(cr) edge[->,dashed,out=245,in=280] node{\bf\scriptsize \ref{r2}}(d)

(h) edge[->,dashed,out=100,in=300] node{\bf\scriptsize \ref{r3}}(ddd)
(h) edge[->,dashed,out=100,in=300] node{\bf\scriptsize \ref{r3}}(rs2)

(h) edge[->,dashed,out=250,in=300] node{\bf\scriptsize \ref{r4}}(cs2)

(dd) edge[->,dashed,out=250,in=300] node{\bf\scriptsize \ref{r5}}(r)

(dd) edge[->,dashed,out=250,in=300] node{\bf\scriptsize \ref{r6}}(c)

(s2) edge[->,dashed,out=250,in=300] node{\bf\scriptsize \ref{r7}}(d)
(s2) edge[->,dashed,out=150,in=340] node{\bf\scriptsize \ref{r7}}(cr)
(rs2) edge[->,dashed,out=150,in=340] node{\bf\scriptsize \ref{r7}}(crr)

(crs) edge[->,dashed,out=320,in=140] node{\bf\scriptsize \ref{r8}}(ss)
(crrs) edge[->,dashed,out=320,in=150] node{\bf\scriptsize \ref{r8}}(rss)

(crrss) edge[->] (crrs)
(crrss) edge[->] (crrs2)
(crrss) edge[->] (crss)
(crrss) edge[->] (rrss)

(crrs) edge[->] (crr)
(crrs) edge[->] (crs)
(crrs) edge[->] (rrs)

(crrs2) edge[->] (crr)
(crrs2) edge[->] (crs2)
(crrs2) edge[->] (rrs2)

(crss) edge[->] (crs)
(crss) edge[->] (crs2)
(crss) edge[->] (css)
(crss) edge[->] (rss)

(rrss) edge[->] (rrs)
(rrss) edge[->] (rrs2)
(rrss) edge[->] (rss)

(crr) edge[->] (cr)
(crr) edge[->] (rr)

(crs) edge[->] (cr)
(crs) edge[->] (cs)

(crs) edge[->] (rs)

(crs2) edge[->] (cs2)
(crs2) edge[->] (cr)
(crs2) edge[->] (rs2)

(css) edge[->] (cs)
(css) edge[->] (cs2)
(css) edge[->] (ss)

(rrs) edge[->] (rs)
(rrs) edge[->] (rr)

(rrs2) edge[->] (rs2)
(rrs2) edge[->] (rr)

(rss) edge[->] (rs)
(rss) edge[->] (rs2)
(rss) edge[->] (ss)

(cr) edge[->] (c)
(cr) edge[->] (r)

(cs) edge[->] (c)
(cs) edge[->] (s)

(cs2) edge[->] (c)
(cs2) edge[->] (s2)

(rr) edge[->] (r)

(rs) edge[->] (r)
(rs) edge[->] (s)

(rs2) edge[->] (r)
(rs2) edge[->] (s2)

(ss) edge[->] (s)
(ss) edge[->] (s2)

(dddh) edge[->] (ddd)
(dddh) edge[->] (ddh)

(ddd) edge[->] (dd)
(ddh) edge[->] (dd)
(ddh) edge[->] (dh)

(dd) edge[->] (d)
(dh) edge[->] (d)
(dh) edge[->] (h)

;

\end{tikzpicture}
\end{minipage}}
\caption{Graphical representation of the preference relation induced by the $\mci$-net in Example \ref{ex:sp2}, which encodes confined reasoning on the $\pci$-net in Example \ref{ex:sp1} w.r.t. the benefits appearing in Table \ref{tb:18}.  Solid arcs are obtained by monotonicity, dotted arcs are obtained via CI-statements (and monotonicity in some cases).  Transitivity arcs are ommited.  Only sets of benefits which all result from the \emph{same} treatment according to the evidence in Table \ref{tb:18} are shown.}
\label{fig:11}
\end{figure}
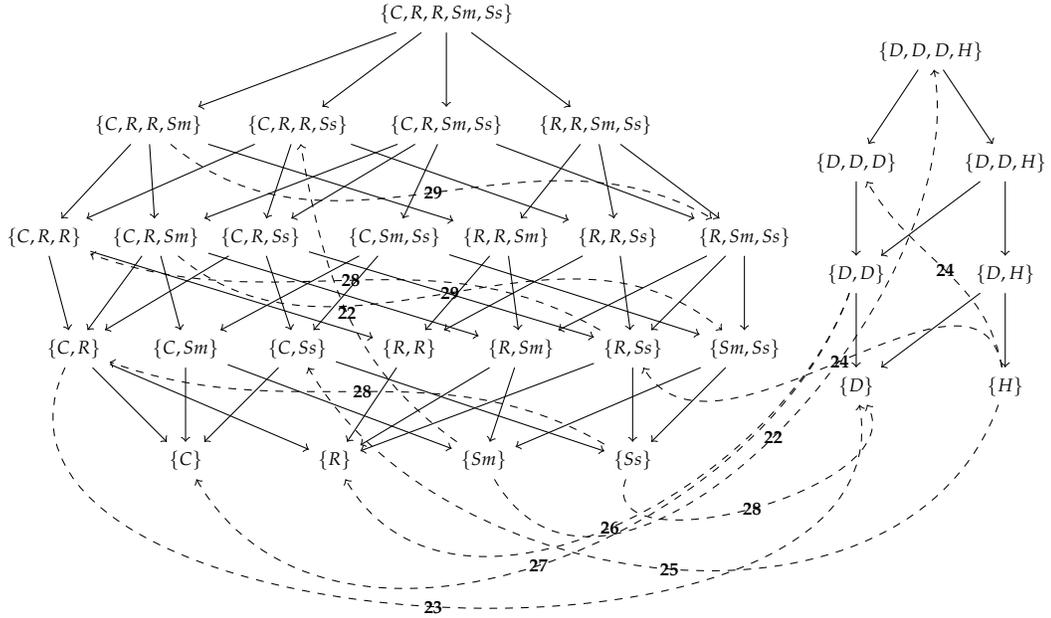

Now, one can take c-consistency w.r.t. $M$ as an additional indicator to trust the specification of preferences in Example \ref{ex:sp1} is consistent (i.e. the $\pci$-net is satisfiable).  Under this assumption, since there is a sequence of worsening flips even from $\{Sm\}$, corresponding to evidence item $e_{01}$ in Table \ref{tb:18}, to $\{D,D,D,H\}$, corresponding to the evidence set $\{e_{06},e_{07},e_{08},e_{09}\}$ in Table \ref{tb:18}, one can conclude from Figure \ref{fig:11} that prostaglandin analogue is preferred to beta-blockers as a treatment for patients who have raised intraocular pressure.  Note, nevertheless that if, for example, one were to remove evidence item $e_{01}$ from Table \ref{tb:18} (e.g. because of statistical significance) then one would not be able to reach a decision between the treatments based on preferences alone.

\section{Conclusion \& future work}
\label{sec:conc}

In this work we presented some initial ideas on how to build a framework for encoding monotonic preferences over multiset of goods on top of $\sci$-nets in the sense that at least a restricted form of reasoning, which we called confined reasoning, can be efficiently reduced to reasoning about $\sci$-nets.  To the best of our knowledge this is the first work considering ordinal multiset preferences, certainly in the context of $\sci$-nets.

Further investigation is required regarding the adequacy of $\pci$-nets (and $\sci$-nets for that manner) for encoding multi-set preferences over goods.  We have shown that $\pci$-nets can be useful in some contexts, such as for encoding some of the purely qualitative preferences we considered in Section \ref{sec:pci} as well as a component of the system of aggregation of evidence stemming from clinical trials by \citeauthor{HunterW12}.  Nevertheless, we have also shown some limitations in our consideration of purely qualitative preferences and it remains to be seen if more complex examples in the context of evidence aggregation can be easily encoded using $\pci$-nets. Some progress may be achieved by introducing a further abstraction layer over $\pci$-nets.  

We have built our framework for encoding multi-set preferences on $\sci$-nets because of our interest in preferences over multi-sets of goods, which are monotonic.  But it would be interesting to explore encoding multi-set preferences using ideas from other formalisms which have been proposed for encoding preferences over sets, in particular that of \citeauthor{BrewkaTW10}(\citeyear{BrewkaTW10}).

In this work we also laid the basic ground-work for computational procedures for reasoning about $\pci$ nets via our characterisation of reasoning about $\pci$-nets in terms of confined reasoning and the reduction of the latter, first to reasoning about $\mci$-nets, and then to $\sci$-nets.  At the very least, our results allow for sound and complete procedures for confined reasoning about $\pci$-nets.  

We note that this situation is in practice less clearly different to that of $\sci$-nets and related formalisms (such as CP nets) than may appear at first glance.  The reason is that the complexity of reasoning about $\sci$ and CP nets (see \cite{GoldsmithLTW08} for the latter) means that one will also for these formalisms usually have to rely on methods which, for example, are either incomplete or only complete for a restricted subset of the formalisms.  Together with the fact that often specifications of preference relations induce preference relations which are themselves incomplete, this situation provides an additional reason for complementing preferences with other techniques, such as is argumentation in the system of \citeauthor{HunterW12}, for purposes of decision making.    

Nevertheless, the question of finding a multi-set that captures the preference relation for a particular evaluation context is an important question that remains largely unexplored in our work.  Also, finding subclasses of $\pci$-nets beyond acyclic ones (or those stemming from $\sci$-nets) where such a multi-set can be found or where satisfiability can be guaranteed are important questions.  Likewise, complexity issues  remain to be explored (in particular, for subclasses of $\pci$ and $\sci$ nets; we note that several results regarding $\sci$-nets, such as that for acyclic nets and also for SCI-nets, may be lifted to $\pci$-nets).

Equally relevant is considering adapting techniques for reasoning about CP nets ``in practice'' such as are considered in \cite{Allen15} (some of which are also discussed for $\sci$-nets in \cite{BouveretEL09}) to the $\pci$-net scenario.  Finally, computational procedures and systems for $\sci$-nets such as presented in \cite{SanthanamBH10,SanthanamBH15,2016Santhanam} can be adapted to $\mci$-nets or, alternatively, be optimised for $\sci$-nets resulting from a reduction from $\mci$-nets.

\clearpage
\newpage

\bibliographystyle{apa}
\bibliography{references}

\end{document}